\documentclass[11pt,twoside]{article}
\usepackage{hyperref,xspace,fullpage,amsmath,amsthm,amssymb,amsfonts,float}
\usepackage{algorithm,algorithmic}
\floatstyle{ruled}
\newfloat{algorithm}{h}{loa}
\floatname{algorithm}{Algorithm}
\newcommand{\err}{\mathsf{err}}
\newcommand{\cp}{\mathsf{cp}}

\newcommand{\on}{\{-1,1\}}

\newcommand{\anglesep}{\theta}

\usepackage{amssymb}
\usepackage{amsmath}

\newcommand{\eg}{e.g.\ }

\newcommand{\fr}[1]{\frac{1}{#1}}

\newcommand{\A}{{\mathcal A}}

\newcommand{\B}{{\mathcal B}}

\newcommand{\C}{{\mathcal C}}

\newcommand{\cH}{{\mathcal H}}

\newcommand{\R}{{\mathbb R}}
\newcommand{\E}{\mathop{\mathbf E}}

\newcommand{\pr}{\mathop{\mathbf{Pr}}}
\newcommand{\ra}{\rangle}
\newcommand{\la}{\langle}
\newcommand{\cond}{\ |\ }

\newcommand{\etal}{{\em et al.\ }}

\newcommand{\poly}{\mbox{poly}}

\newcommand{\ti} \tilde

\newcommand{\sgn}{\mathsf{sign}}
\newcommand{\eps}{\epsilon}

\newcommand{\equ}[1]{

\begin{equation}
#1
\end{equation}}
\newcommand{\equn}[1]{
$$
#1
$$}

\newcommand{\ignore}[1]{\relax}
\newcommand{\alequ}[1]{\begin{align} #1 \end{align}}
\newcommand{\alequn}[1]{\begin{align*} #1 \end{align*}}

\newcommand{\eat}[1]{}

\newtheorem{theorem}{Theorem}[section]
\newtheorem{lemma}[theorem]{Lemma}

\newtheorem{corollary}[theorem]{Corollary}
\newtheorem{remark}[theorem]{Remark}

\newtheorem{definition}[theorem]{Definition}

\renewcommand{\A}{ P} % also probability

\newcommand{\DIS}{{\rm DIS}}

\begin{document}

\title{Statistical Active Learning Algorithms \\for Noise Tolerance and Differential Privacy\\}

\author{
Maria Florina Balcan \\
Carnegie Mellon University \\
{\tt ninamf@cs.cmu.edu}
% 2nd. author
\and
Vitaly Feldman \\
IBM Research - Almaden  \\
{\tt vitaly@post.harvard.edu}
}
\date{}

\maketitle

\begin{abstract}
We describe a framework for designing efficient active learning algorithms that are tolerant to random classification noise and are differentially-private. The framework is based on active learning algorithms that are {\em statistical} in the sense that they rely on estimates of expectations of functions of filtered random examples. It builds on the powerful statistical query framework of Kearns \cite{Kearns:98}.

We show that any efficient active statistical learning algorithm can be automatically converted to an efficient active learning algorithm which is tolerant to random classification noise as well as other forms of ``uncorrelated" noise. The complexity of the resulting algorithms has information-theoretically optimal quadratic dependence on $1/(1-2\eta)$, where $\eta$ is the noise rate.

We show that commonly studied concept classes including thresholds, rectangles, and linear separators can be efficiently actively learned in our framework. These results combined with our generic conversion lead to the first computationally-efficient algorithms for actively learning some of these concept classes in the presence of random classification noise that provide exponential improvement in the dependence on the error $\eps$ over their passive counterparts. In addition, we show that our algorithms can be automatically converted to efficient active differentially-private algorithms. This leads to the first differentially-private active learning algorithms with exponential label savings over the passive case.
\end{abstract}

\section{Introduction}
Most classic machine learning methods depend on the assumption that humans can annotate all the data available for training. However, many modern machine learning applications %%(including video classification and protein sequence classification)
have massive amounts of unannotated or unlabeled data.
As a consequence, there has been tremendous interest both in machine learning and its application areas in designing algorithms that most efficiently utilize the available data, while minimizing the need for human intervention.
  An extensively used and studied technique is active learning, where the algorithm is presented with a large pool of unlabeled
examples %% (such as all images available on the web)
and can interactively ask for the labels of examples
of its own choosing from the pool, with the goal to drastically reduce labeling effort.
This has been a major area of machine learning research in the past decade~\cite{sanjoy11-encyc,hanneke:survey}, with several exciting developments  on understanding its underlying statistical principles~\cite{QBC,sanjoy-coarse,BBL06,BalcanBZ:07,Hanneke07,dhsm,CN07,BHW08,Kol10,nips10,wang11,RaginskyR:11,BH12}.
In particular, several general characterizations have been developed for
describing when
active learning can in principle have an advantage over the classic passive supervised learning paradigm, and by how much.
While the %sample
label
complexity aspect of active learning has been intensively studied and is currently well understood, the question of providing computationally efficient noise tolerant active learning algorithms has remained largely open.
In particular, prior to this work, there were no known efficient active algorithms for concept classes of super-constant VC-dimension that are provably robust to random and independent noise while giving improvements over the passive case.

\subsection{Our Results}
%%%add reduction in the intro...
We propose a framework for designing efficient (polynomial time) active learning algorithms which is based on restricting the way in which examples (both labeled and unlabeled) are accessed by the algorithm. These restricted algorithms can be easily simulated using active sampling and, in addition, possess a number of other useful properties. The main property we will consider is tolerance to random classification noise of rate $\eta$ (each label is flipped randomly and independently with probability $\eta$ \cite{AngluinLaird:88}). Further, as we will show, the algorithms are tolerant to other forms of noise and can be simulated in a differentially-private way.
%%%tolerance while maintainig efficiency....

In our restriction, instead of access to random examples from some distribution $\A$ over $X \times Y$ the learning algorithm only gets ``active" estimates of the statistical properties of $\A$ in the following sense. The algorithm can choose any {\em filter} function $\chi(x): X \rightarrow [0,1]$ and a query function $\phi: X \times Y \rightarrow [-1,1]$ for any $\chi$ and $\phi$. For simplicity we can think of $\chi$ as an indicator function of some set $\chi_S \subseteq X$ of ``informative" points and of $\phi$ as some useful property of the target function. For this pair of functions the learning algorithm can get an estimate of $\E_{(x,y) \sim \A} [\phi(x,y) \cond x \in \chi_S]$.
For $\tau$ and $\tau_0$ chosen by the algorithm the estimate is provided to within {\em  tolerance} $\tau$ as long as $\E_{(x,y) \sim \A} [x \in \chi_S] \geq \tau_0$ (nothing is guaranteed otherwise). The key point it that when we simulate this query from random examples, the inverse of $\tau$ corresponds to the label complexity of the algorithm and the inverse of $\tau_0$ corresponds to its unlabeled sample complexity. Such a query is referred to as {\em active statistical query (SQ)} and algorithms using active SQs are referred to as {\em active statistical algorithms}.

Our framework builds on the classic statistical query (SQ) learning framework of Kearns~\cite{Kearns:98} defined in the context of PAC learning model~\cite{Valiant:84}. The SQ model is based on estimates of expectations of functions of examples (but without the additional filter function) and was defined in order to design efficient noise tolerant algorithms in the PAC model. Despite the restrictive form, most of the learning algorithms in the PAC model and other standard techniques in machine learning and statistics used for problems over distributions have SQ analogues \cite{Kearns:98,BlumFKV:97,BlumDMN:05,ChuKLYBNO:06,FeldmanGRVX:13}\footnote{The sample complexity of the SQ analogues might be polynomially larger though.}. Further, statistical algorithms enjoy additional properties: they can be simulated in a differentially-private way \cite{BlumDMN:05}, automatically parallelized on multi-core architectures \cite{ChuKLYBNO:06} and have known information-theoretic characterizations of query complexity \cite{BlumFJ+:94,Feldman:12jcss}. As we show, our framework inherits the strengths of the SQ model while, as we will argue, capturing the power of active learning.

 At a first glance being active and statistical appear to be incompatible requirements on the algorithm. Active algorithms typically make label query decisions on the basis of examining individual samples (for example as in binary search for learning a threshold or the algorithms in \cite{QBC,dhsm,DasguptaKM:09}). At the same time statistical algorithms can only examine properties of the underlying distribution. But there also exist a number of active learning algorithms that can be seen as applying passive learning techniques to batches of examples that are obtained from querying labels of samples that satisfy the same filter. These include the general $A^2$ algorithm ~\cite{BBL06} and, for example, algorithms in \cite{BalcanBZ:07,DasguptaHsu:08,BeygelzimerDL:09,BalcanLong:13}.
 As we show, we can build on these techniques to provide algorithms that fit our framework.

We start by presenting a general reduction showing that any efficient active statistical learning algorithm can be automatically converted to an efficient active learning algorithm which is tolerant to random classification noise as well as other forms of ``uncorrelated" noise.
The sample complexity of the resulting algorithms depends just quadratically on $1/(1-2\eta)$, where $\eta$ is the noise rate.

%%%
We then demonstrate the generality of our framework by showing that the most commonly studied concept classes including thresholds, balanced rectangles, and homogenous linear separators can be efficiently actively learned via active statistical algorithms. For these concept classes, we design efficient active learning algorithms that are statistical and provide the same exponential improvements in the dependence on the error $\eps$ over passive learning as their non-statistical counterparts.

The primary problem we consider is active learning of homogeneous halfspaces a problem that has attracted a lot of interest in the theory of active learning \cite{QBC,sanjoy-coarse,BalcanBZ:07,BeygelzimerDL:09,DasguptaKM:09,CaCEGe10,dgs12,BalcanLong:13,GSS12}.
We describe two algorithms for the problem. First, building on insights from margin based analysis of active learning~\cite{BalcanBZ:07,BalcanLong:13}, we give an active statistical learning algorithm for homogeneous halfspaces over all isotropic log-concave distributions, a wide class of distributions that includes many well-studied density functions and has played an important role in several areas including sampling, optimization, integration, and learning \cite{LovaszVempala:07}.
Our algorithm for this setting proceeds in rounds;  in round $t$ we build a better
approximation $w_t$ to the target function by using a passive SQ
learning algorithm (e.g., the one of~\cite{DunaganVempala:04}) over a distribution $D_t$
that is a mixture of distributions in which each component is the
original distribution conditioned on being within a certain distance
from the hyperplane defined by previous approximations $w_i$. To
perform passive statistical queries relative to $D_t$ we use active
SQs with a corresponding real valued filter.\eat{
\footnote{Note that the earlier
margin-based algorithms analyzed in~\cite{BalcanBZ:07,BalcanLong:13}
 operate by only querying points close to the  hypothesis
$w_t$ in round $t$. As a result the analysis of our
algorithm is somewhat different from that in earlier work
~\cite{BalcanBZ:07,BalcanLong:13}.}}
This algorithm is computationally efficient and uses only $\poly(d,\log(1/\epsilon))$ active statistical queries of tolerance inverse-polynomial in the dimension $d$ and $\log(1/\epsilon)$.

For the special case of the uniform distribution over the unit ball we give a new, simpler and substantially more efficient active statistical learning algorithm. Our algorithm is based on measuring the error of a halfspace conditioned on being within some margin of that halfspace. We show that such measurements performed on the perturbations of the current hypothesis along the $d$ basis vectors can be combined to derive a better hypothesis. This approach differs substantially from the previous algorithms for this problem \cite{BalcanBZ:07,DasguptaKM:09}. The algorithm is computationally efficient and uses $d\log(1/\eps)$ active SQs with tolerance of $\Omega(1/\sqrt{d})$ and filter tolerance of $\Omega(\eps)$. %%% and runs in time $d \cdot \poly(\log{(d/\eps)})$.

These results, combined with our generic simulation of active statistical algorithms in the presence of random classification noise (RCN) lead to the first known computationally efficient algorithms for actively learning halfspaces which are RCN tolerant and give provable label savings over the passive case. For the uniform distribution case this leads to an algorithm with sample complexity of  $O((1-2\eta)^{-2} \cdot d^2 \log(1/\eps) \log(d \log(1/\eps)))$ and for the general isotropic log-concave case we get sample complexity of $\poly(d,\log(1/\eps),1/(1-2\eta))$. This is worse than the sample complexity in the noiseless case which is just $O((d+\log\log(1/\eps)) \log(1/\eps))$ \cite{BalcanLong:13}. However, compared to passive learning in the presence of RCN, our algorithms have exponentially better dependence on $\eps$ and essentially the same dependence on $d$ and $1/(1-2\eta)$. One issue with the generic simulation is that it requires knowledge of $\eta$ (or an almost precise estimate). Standard approach to dealing with this issue does not always work in the active setting and for our  log-concave and the uniform distribution algorithms we give a specialized argument that preserves the exponential improvement in the dependence on $\eps$.

\smallskip
\noindent {\bf Differentially-private active learning:} In many application of machine learning such as medical and financial record analysis, data is both sensitive and expensive to label. However, to the best of our knowledge, there are no formal results addressing both of these constraints. We address the problem by defining a natural model of differentially-private active learning. In our model we assume that a learner has full access to unlabeled portion of some database of $n$ examples $S \subseteq X \times Y$ which correspond to records of individual participants in the database. In addition, for every element of the database $S$ the learner can request the label of that element. As usual, the goal is to minimize the number of label requests (such setup is referred to as {\em pool-based} active learning \cite{McNi98}). In addition, we would like to preserve the {\em differential privacy} of the participants in the database, a now-standard notion of privacy introduced in \cite{DMNS06}. Informally speaking, an algorithm is differentially private if adding any record to $S$ (or removing a record from $S$) does not affect the probability that any specific hypothesis will be output by the algorithm significantly.

%Finally, we show that active SQ learning algorithms can also be used to obtain differentially-private active learning algorithms.

As first shown by Blum \etal \cite{BlumDMN:05}, SQ algorithms can be automatically translated into differentially-private algorithms by using the so-called Laplace mechanism (see also \cite{KasiviswanathanLNRS11}). Using a similar approach, we show that active SQ learning algorithms can be automatically transformed into differentially-private active learning algorithms. As a consequence, for all the classes for which we provide statistical active learning algorithms that can be simulated by using only $\poly(d,\log(1/\epsilon))$ labeled examples (including thresholds  and halfspaces),  we can learn and preserve privacy with much fewer label requests than those required by even non-private classic passive learning algorithms, and can do so even when in our model the privacy parameter is very small. Note that while we focus on the number of label requests, the algorithms also preserve the differential privacy of the unlabeled points.
%Using our active statistical algorithms for halfspaces we obtain the first algorithms that are both differentially-private and give exponential improvements %in the dependence of label complexity on the accuracy  parameter $\eps$.

 \subsection{Additional Related Work}
 \label{sec:related-work}
As we have mentioned, most prior theoretical work on active learning focuses on either sample complexity bounds (without regard for efficiency) or the noiseless case.
For random classification noise in particular,~\cite{BH12} provides a sample complexity analysis based on the notion of splitting index that is optimal up to $\mathrm{polylog}$ factors and works for general concept classes and distributions, but it is not computationally efficient.
 In addition, several works give active learning algorithms with empirical evidence of robustness to certain types of noise~\cite{BeygelzimerDL:09,GSS12}; %%%additionally,~\cite{GSS12} provide guarantees for  noisy scenarios that can tolerate ....
 %%% be more specific here

In~\cite{CaCEGe10,dgs12} online learning algorithms in the selective sampling framework are presented, where labels must
be actively queried before they are revealed. Under the assumption that the label conditional distribution is a linear function determined by a fixed target vector, they provide bounds on the regret of the algorithm and on
the number of labels it queries when faced with an adaptive adversarial strategy of generating the
instances. As pointed out in \cite{dgs12}, these results can also be converted to a distributional PAC setting where instances $x_t$ are drawn i.i.d. In this setting they obtain exponential improvement in label complexity over passive learning. These interesting results and techniques are not directly comparable to ours. Our framework is not restricted to halfspaces. Another important difference is that (as pointed out in~\cite{GSS12}) the exponential improvement they give is not possible in the noiseless version of their setting. In other words, the addition of linear noise defined by the target makes the problem easier for active sampling. By contrast RCN can only make the classification task harder than in the realizable case.

Among the so called disagreement-based algorithms that provably work under very general noise models (adversarial label noise) and for general concept classes~\cite{BBL06,Kol10,dhsm,nips10,wang11,RaginskyR:11,BH12,hanneke:survey}, those of Dasgupta, Hsu, and Monteleoni~\cite{dhsm} and Beygelzimer, Hsu, Langford, and Zhang~\cite{nips10} are most  amenable to implementation.
%%%% better label complexity
 While more amenable to implementation than other disagreement-based techniques, these algorithms assume the existence of a computationally efficient passive learning algorithm (for the concept class at hand) that can minimize the empirical error in the adversarial label noise --- however, such algorithms are not known to exist for most concept classes, including linear separators.

Following the original publication of our work,  Awasthi \etal~\cite{ABL-14} give a polynomial-time
active learning algorithm for learning linear separators in the presence of adversarial forms of noise.
Their algorithm is the first one that can tolerate both adversarial label noise and malicious noise (where the adversary can corrupt both the instance part and the label part of the examples) as long as the rate of noise $\eta=O(\epsilon)$. We note that these results are not comparable to ours as we need the noise to be ``uncorrelated" but can deal with noise of any rate (with complexity growing with $1/(1-2\eta)$).

%amount of noise as long as the number of  samples and running time increases  proportionally to the noise rate.
\eat{
We note that subsequent to the conference publication of our work, Awasthi, Balcan, and Long~\cite{ABL-arxiv} provided the first polynomial-time
 active learning algorithm for learning linear
separators in the presence of adversarial label noise, as well as the first analysis of
active learning under the challenging malicious noise where the adversary can corrupt both the instance part and the label part of our examples. We note that these results are  incomparable to ours. Because the noise models they consider are less structured,~\cite{ABL-arxiv} can deal with noise rates $\eta=O(\epsilon)$; by contrast, in our case, for the more stylized  random classification noise,  we can deal with any amount of noise as long as the number of  samples and running time increases  proportionally to the noise rate.
}
\smallskip

\noindent {\bf Organization:} Our model, its properties and several illustrative examples (including threshold functions and balanced rectangles) are given in Section~\ref{sec:model}. Our algorithm for learning homogeneous halfspaces over log-concave and uniform distributions
are given in Section~\ref{sec:hs-logc} and Section~\ref{sec:uniform} respectively. The formal statement of differentially-private simulation is given in Section~\ref{sec:privacy}.

 %% expand

%\input{asq-model.2.tex}

\section{Active Statistical Algorithms}
\label{sec:model}
Let $X$ be a domain and $\A$ be a distribution over labeled examples on $X$. We represent such a distribution by a pair $(D,\psi)$ where $D$ is the marginal distribution of $\A$ on $X$ and $\psi:X \rightarrow [-1,1]$ is a function defined as $\psi(z) = \E_{(x,\ell) \sim \A}[\ell \cond x = z]$. We will be considering learning in the PAC model (realizable case) where $\psi$ is a boolean function, possibly corrupted by random noise.

When learning with respect to a distribution $\A=(D,\psi)$, an active statistical learner has access to
{\em active statistical queries}.
 A query of this type is a pair of functions $(\chi, \phi)$, where $\chi: X \rightarrow [0,1]$ is the {\em filter} function which for a point $x$, specifies the probability with which the label of $x$ should be queried. The function $\phi: X \times \on \rightarrow [-1,1]$ is the query function and depends on both point and the label. The filter function $\chi$ defines the distribution $D$ conditioned on $\chi$ as follows: for each $x$ the density function $D_{|\chi}(x)$ is defined as $D_{|\chi}(x) = D(x)\chi(x)/\E_D[\chi(x)]$. Note that if $\chi$ is an indicator function of some set $S$ then $D_{|\chi}$ is exactly $D$ conditioned on $x$ being in $S$. Let $\A_{|\chi}$ denote the conditioned distribution $(D_{|\chi},\psi)$. In addition, a query has two tolerance parameters: filter tolerance $\tau_0$ and query tolerance $\tau$. In response to such a query the algorithm obtains a value $\mu$ such that if $\E_D[\chi(x)] \geq \tau_0$ then $$\left|\mu - \E_{\A_{|\chi}}[\phi(x,\ell)] \right| \leq \tau$$ (and nothing is guaranteed when $\E_D[\chi(x)] < \tau_0$).

An active statistical learning algorithm can also ask {\em target-independent} queries with tolerance $\tau$ which are just queries over unlabeled samples. That is for a query $\varphi:X \rightarrow [-1,1]$ the algorithm obtains a value $\mu$, such that $|\mu - \E_D[\varphi(x)]| \leq \tau$. Such queries are not necessary when $D$ is known to the learner.

For the purposes of obtaining noise tolerant algorithms one can relax the requirements of model and give the learning algorithm access to unlabelled samples. A similar variant of the model was considered in the context of SQ model \cite{Kearns:98,BlumFKV:97}. We refer to this variant as {\em label-statistical}. Label-statistical algorithms do not need access to target-independent queries access as they can simulate those using unlabelled samples.

Our definition generalizes the statistical query framework of Kearns \cite{Kearns:98} which does not include filtering function, in other words a query is just a function $\phi: X \times \on \rightarrow [-1,1]$ and it has a single tolerance parameter $\tau$.
By definition, an active SQ $(\chi, \phi)$ with tolerance $\tau$ relative to $\A$ is the same as a passive statistical query $\phi$ with tolerance $\tau$ relative to the distribution $\A_{|\chi}$. In particular, a (passive) SQ is equivalent to an active SQ with filter $\chi \equiv 1$ and filter tolerance $1$.

Finally we note that from the definition of active SQ we can see that
$$ \E_{\A_{|\chi}}[\phi(x,\ell)] = \E_{\A}[\phi(x,\ell) \cdot \chi(x)] / \E_{\A}[\chi(x)].$$
This implies that an active statistical query can be estimated using two passive statistical queries. However to estimate $\E_{\A_{|\chi}}[\phi(x,\ell)]$ with tolerance $\tau$ one needs to estimate $\E_{\A}[\phi(x,\ell) \cdot \chi(x)]$ with tolerance $\tau \cdot \E_{\A}[\chi(x)]$ which can be much lower than $\tau$. Tolerance of a SQ directly corresponds to the number of examples needed to evaluate it and therefore simulating active SQs passively might require many more examples.

\subsection{Simulating Active Statistical Queries}
In our model, the algorithm operates via statistical queries. In this section we describe how the answers to these queries can be simulated from random examples, which immediately implies that our algorithms can be transformed into active learning algorithms in the usual model~\cite{sanjoy11-encyc}.

We first note that a valid response to a target-independent query with tolerance $\tau$ can be obtained, with probability at least $1-\delta$, using $O(\tau^{-2}\log{(1/\delta)})$ unlabeled samples.

A natural way of simulating an active SQ is by filtering points drawn randomly from $D$: draw a random point $x$, let $B$ be drawn from Bernoulli distribution with probability of $1$ being $\chi(x)$; ask for the label of $x$ when $B=1$. The points for which we ask for a label are distributed according to $D_{|\chi}$. This implies that the empirical average of $\phi(x,\ell)$ on $O(\tau^{-2}\log{(1/\delta)})$ labeled examples will then give $\mu$. Formally we get the following theorem.
\begin{theorem}
\label{th:simulate}
Let $\A=(D,\psi)$ be a distribution over $X \times \on$. There exists an active sampling algorithm that given functions $\chi: X \rightarrow [0,1]$, $\phi: X \times \on \rightarrow [-1,1]$, values  $\tau_0 > 0$, $\tau > 0$, $\delta > 0$, and access to samples from $\A$, with probability at least $1-\delta$, outputs a valid response to active statistical query $(\chi, \phi)$ with tolerance parameters $(\tau_0,\tau)$. The algorithm uses $O(\tau^{-2}\log{(1/\delta)})$ labeled examples from $\A$ and $O(\tau_0^{-1}\tau^{-2}\log{(1/\delta)})$ unlabeled samples from $D$.
\end{theorem}
\begin{proof}
The Chernoff-Hoeffding bounds imply that for some $t = O(\tau^{-2}\log{(1/\delta)})$, the empirical mean of $\phi$ on $t$ examples that are drawn randomly from $\A_{|\chi}$ will, with probability at least $1-\delta/2$, be within $\tau$ of $\E_{\A_{|\chi}}[\phi(x,\ell)]$. We can also assume that $\E_D[\chi(x)] \geq \tau_0$ since any value would be a valid response to the query when this assumption does not hold. By the standard multiplicative form of the Chernoff bound we also know that given $t_0 =  O(\tau_0^{-1} t \cdot \log{(1/\delta)}) = O(\tau_0^{-1}\tau^{-2}\log{(1/\delta)}^2)$ random samples from $D$, with probability at least $1-\delta/2$, at least $t$ of the samples will pass the filter $\chi$. Therefore with, probability at least $1-\delta$, we will obtain at least $t$ samples from $D$ filtered using $\chi(x)$ and labeled examples on these points will give an estimate of $\E_{\A_{|\chi}}[\phi(x,\ell)]$ with tolerance $\tau$.

This procedure gives $\log{(1/\delta)}^2$ dependence on confidence (and not the claimed $\log{(1/\delta)}$). To get the claimed dependence we can use a standard confidence boosting technique. We run the above procedure with $\delta' = 1/3$, $k$ times and let $\mu_1,\mu_2,\ldots, \mu_k$ denote the results. The simulation returns the median of $\mu_i$'s. The Chernoff bound implies that for $k=O(\log(1/\delta))$, with probability at least $1-\delta$, at least half of the $\mu_i$'s satisfy the condition $\left|\mu_i - \E_{\A_{|\chi}}[\phi(x,\ell)] \right| \leq \tau$. In particular, the median satisfies this condition. The dependence on $\delta$ of sample complexity is now as claimed.
\end{proof}
We remark that in some cases better sample complexity bounds can be obtained using multiplicative forms of the Chernoff-Hoeffding bounds (\eg \cite{AslamDecatur:98}).

A direct way to simulate all the queries of an active SQ algorithm is to estimate the response to each query using fresh samples and use the union bound to ensure that, with probability at least $1-\delta$, all queries are answered correctly. Such direct simulation of an algorithm that uses at most $q$ queries can be done using $O(q\tau^{-2} \log(q/\delta))$ labeled examples and $O(q\tau_0^{-1}\tau^{-2}\log{(q/\delta)})$ unlabeled samples. However, in many cases a more careful analysis can be used to reduce the sample complexity of simulation. Labeled examples can be shared to simulate queries that use the same filter $\chi$ and do not depend on each other. This implies that the sample size sufficient for simulating $q$ non-adaptive queries with the same filter scales logarithmically with $q$. More generally, given a set of $q$ query functions (possibly chosen adaptively) which belong to some set $Q$ of low complexity (such as VC dimension) one can reduce the sample complexity of estimating the answers to all $q$ queries (with the same filter) by invoking the standard bounds based on uniform convergence (\eg \cite{BlumerEH+:89,vapnik:98}). %The direct implication of this is that an active SQ algorithm that uses at most $q$ active SQs of tolerance at most $\tau$ conditioned relative on each of at most $k$ filters can be simulated using $O(k \cdot \tau^{-2} \log (q/\delta))$ labeled examples.

\subsection{Noise tolerance}
An important property of the simulation described in Theorem \ref{th:simulate} is that it can be easily adapted to the case when the labels are corrupted by random classification noise \cite{AngluinLaird:88}.  For a distribution $\A=(D,\psi)$ let $\A^\eta$ denote the distribution $\A$ with the label flipped with probability $\eta$ randomly and independently of an example. It is easy to see that $\A^\eta = (D, (1-2\eta) \psi)$. We now show that, as in the SQ model \cite{Kearns:98}, active statistical queries can be simulated given examples from $\A^\eta$.
\begin{theorem}
\label{th:simulate-noise}
Let $\A=(D,\psi)$ be a distribution over examples and let $\eta \in [0,1/2)$ be a noise rate. There exists an active sampling algorithm that given functions $\chi: X \rightarrow [0,1]$, $\phi: X \times \on \rightarrow [-1,1]$, values $\eta$, $\tau_0 > 0$, $\tau > 0$, $\delta > 0$, and access to samples from $\A^\eta$, with probability at least $1-\delta$, outputs a valid response to active statistical query $(\chi, \phi)$ with tolerance parameters $(\tau_0,\tau)$. The algorithm uses $O(\tau^{-2}(1-2\eta)^{-2}\log{(1/\delta)})$ labeled examples from $\A^\eta$ and $O(\tau_0^{-1}\tau^{-2}(1-2\eta)^{-2}\log{(1/\delta)})$ unlabeled samples from $D$.
\end{theorem}
\begin{proof}
Using a simple observation from \cite{BshoutyFeldman:02}, we first decompose the statistical query $\phi$ into two parts: one that computes a correlation with the label and the other that does not depend on the label altogether. Namely,
\equ{\phi(x,\ell) = \phi(x,1) \frac{1 +\ell}{2} + \phi(x,-1) \frac{1 - \ell}{2} = \frac{\phi(x,1)  - \phi(x,-1)}{2} \cdot \ell +
\frac{\phi(x,1)  + \phi(x,-1)}{2}\ .\label{eq:query-decompose}} Clearly, to estimate the value of $\E_{\A_{|\chi}}[\phi(x,\ell)]$ with tolerance $\tau$ it is sufficient to estimate the values of $\E_{\A_{|\chi}}[\frac{1}{2} (\phi(x,1) - \phi(x,-1)) \cdot \ell]$ and $\E_{\A_{|\chi}}[\frac{1}{2} (\phi(x,1) + \phi(x,-1))]$ with tolerance $\tau/2$. The latter expression does not depend on the label and, in particular, is not affected by noise. Therefore it can be estimated as before using $\A^\eta$ in place of $\A$.
At the same time we can use the independence of noise to conclude\footnote{For any function $f(x)$ that does not depend on the label, we have:
$\E_{\A^\eta_{|\chi}}[f(x)\cdot \ell] = (1-\eta)\E_{\A_{|\chi}}[f(x)\cdot \ell] + \eta\cdot \E_{\A_{|\chi}}[f(x)\cdot(-\ell)]
                   = (1-2\cdot\eta)\E_{\A_{|\chi}}[f(x)\cdot \ell]$.
The first equality follows from the fact that under $\A^\eta_{|\chi}$, for any given $x$,
 there is a $(1-\eta)$ chance that the label is the same as under $\A_{|\chi}$, and an $\eta$
 chance that the label is the negation of the label obtained from $\A_{|\chi}$.
},
\equn{\E_{\A^\eta_{|\chi}}\left[\frac{1}{2} (\phi(x,1) - \phi(x,-1)) \cdot \ell \right]
%= \E_{\A^\eta_{|\chi}}[\frac{1}{2} (\phi(x,1) - \phi(x,-1)) ] \cdot \E_{\A^\eta_{|\chi}}[\ell]\\ &
= (1-2\eta) \E_{\A_{|\chi}}\left[\frac{1}{2} (\phi(x,1) - \phi(x,-1)) \cdot \ell \right].}
This means that we can estimate $\E_{\A_{|\chi}}[\frac{1}{2} (\phi(x,1) - \phi(x,-1)) \cdot \ell]$ with tolerance $\tau/2$ by estimating $\E_{\A^\eta_{|\chi}}[\frac{1}{2} (\phi(x,1) - \phi(x,-1)) \cdot \ell]$ with tolerance $(1-2\eta)\tau/2$ and then multiplying the result by $1/(1-2\eta)$. The estimation of $\E_{\A^\eta_{|\chi}}[\frac{1}{2} (\phi(x,1) - \phi(x,-1)) \cdot \ell]$ with tolerance $(1-2\eta)\tau/2$ can be done exactly as in Theorem \ref{th:simulate}.
\end{proof}
Note that the sample complexity of the resulting active sampling algorithm has information-theoretically optimal quadratic dependence on $1/(1-2\eta)$, where $\eta$ is the noise rate.
Note that RCN does not affect the unlabelled samples so algorithms which are only label-statistical algorithms can also be simulated in the presence of RCN.
\begin{remark}
\label{rem:unknown-noise}
This simulation assumes that $\eta$ is given to the algorithm exactly. It is easy to see from the proof, that any value $\eta'$ such that $\frac{1-2\eta}{1-2\eta'} \in [1-\tau/4,1+\tau/4]$ can be used in place of $\eta$ (with the tolerance of estimating $\E_{\A^\eta_{|\chi}}[\frac{1}{2} (\phi(x,1) - \phi(x,-1)) \cdot \ell]$ set to $(1-2\eta)\tau/4$).
In some learning scenarios even an approximate value of $\eta$ is not known but it is known that $\eta \leq \eta_0 < 1/2$. To address this issue one can construct a sequence $\eta_1,\ldots,\eta_k$ of guesses of $\eta$, run the learning algorithm with each of those guesses in place of the true $\eta$ and let $h_1,\ldots,h_k$ be the resulting hypotheses \cite{Kearns:98}. One can then return the hypothesis $h_i$ among those that has the best agreement with a suitably large sample. It is not hard to see that  $k=O(\tau^{-1} \cdot \log(1/(1-2\eta_0)))$ guesses will suffice for this strategy to work \cite{AslamDecatur:98}.

Passive hypothesis testing requires $\Omega(1/\eps)$ labeled examples and might be too expensive to be used with active learning algorithms. It is unclear if there exists a general approach for dealing with unknown $\eta$ in the active learning setting that does not increase substantially the labeled example complexity. However, as we will demonstrate, in the context of specific active learning algorithms variants of this approach can be used to solve the problem.
\end{remark}

We now show that more general types of noise can be tolerated as long as they are ``uncorrelated" with the queries and the target function. Namely, we represent label noise using a function $\Lambda:X \rightarrow [0,1]$, where $\Lambda(x)$ gives the probability that the label of $x$ is flipped. The rate of $\Lambda$ when learning with respect to marginal distribution $D$ over $X$ is $\E_D[\Lambda(x)]$. For a distribution $\A=(D,\psi)$ over examples, we denote by $\A^\Lambda$ the distribution $\A$ corrupted by label noise $\Lambda$. It is easy to see that $\A^\Lambda=(D,\psi \cdot (1-2\Lambda))$. Intuitively, $\Lambda$ is ``uncorrelated" with a query if the way that $\Lambda$ deviates from its rate is almost orthogonal to the query on the target distribution.
\begin{definition}
\label{def:uncorrelated}
Let $\A=(D,\psi)$ be a distribution over examples  and $\tau' > 0$. For functions $\chi: X \rightarrow [0,1]$, $\phi: X \times \on \rightarrow [-1,1]$, we say that a noise function $\Lambda:X \rightarrow [0,1]$ is $(\eta,\tau')$-uncorrelated with $\phi$ and $\chi$ over $\A$ if,
$$\left|\E_{D_{|\chi}}\left[\frac{\phi(x,1) - \phi(x,-1)}{2} \psi(x) \cdot (1-2(\Lambda(x)-\eta)) \right]\right| \leq \tau'\ .$$
\end{definition}
In this definition $(1-2(\Lambda(x)-\eta))$ is the expectation of $\{-1,1\}$ coin that is flipped with probability $\Lambda(x)-\eta$, whereas $(\phi(x,1) - \phi(x,-1)) \psi(x)$ is the part of the query which measures the correlation with the label.
We now give an analogue of Theorem \ref{th:simulate-noise} for this more general setting.
\begin{theorem}
\label{th:simulate-noise-uncorr}
Let $\A=(D,\psi)$ be a distribution over examples, $\chi: X \rightarrow [0,1]$, $\phi: X \times \on \rightarrow [-1,1]$ be a query and a filter functions,  $\eta \in [0,1/2), \tau > 0$ and $\Lambda$ be a noise function that is $(\eta,(1-2\eta)\tau/4)$-uncorrelated with $\phi$ and $\chi$ over $\A$. There exists an active sampling algorithm that given functions $\chi$ and $\phi$, values  $\eta$, $\tau_0 > 0$, $\tau > 0$, $\delta > 0$, and access to samples from $\A^\Lambda$, with probability at least $1-\delta$, outputs a valid response to active statistical query $(\chi, \phi)$ with tolerance parameters $(\tau_0,\tau)$. The algorithm uses $O(\tau^{-2}(1-2\eta)^{-2}\log{(1/\delta)})$ labeled examples from $\A^\Lambda$ and $O(\tau_0^{-1}\tau^{-2}(1-2\eta)^{-2}\log{(1/\delta)})$ unlabeled samples from $D$.
\end{theorem}
\begin{proof}
As in the proof of Theorem \ref{th:simulate-noise}, we note that it is sufficient to estimate the value of
$$\lambda \triangleq \E_{\A_{|\chi}}\left[\frac{1}{2} (\phi(x,1) - \phi(x,-1)) \cdot \ell \right] = \E_{D_{|\chi}}\left[\frac{\phi(x,1) - \phi(x,-1)}{2} \psi(x) \right]$$ within tolerance $\tau/2$ (since $\E_{\A_{|\chi}}[\frac{1}{2} (\phi(x,1) + \phi(x,-1))]$ does not depend on the label and can be estimated as before).
Now
\alequn{ &\E_{\A^\Lambda_{|\chi}}\left[\frac{\phi(x,1) - \phi(x,-1)}{2} \cdot \ell \right] = \E_{D_{|\chi}}\left[\frac{\phi(x,1) - \phi(x,-1)}{2} \cdot \psi(x) \cdot (1-2\Lambda(x)) \right] \\ & = (1-2\eta) \E_{D_{|\chi}}\left[\frac{\phi(x,1) - \phi(x,-1)}{2} \psi(x) \right] + \E_{D_{|\chi}}\left[\frac{\phi(x,1) - \phi(x,-1)}{2} \psi(x) (1-2(\Lambda(x)-\eta)) \right] \\ & =
(1-2\eta) \E_{D_{|\chi}}\left[\frac{\phi(x,1) - \phi(x,-1)}{2} \psi(x) \right] + \tau' = (1-2\eta) \lambda + \tau',
}
where $|\tau'| \leq (1-2\eta)\tau/4$, since $\Lambda$ is $(\eta,(1-2\eta)\tau/4)$-uncorrelated with $\phi$ and $\chi$ over $\A$.

This means that we can estimate $\E_{\A_{|\chi}}[\frac{1}{2} (\phi(x,1) - \phi(x,-1)) \cdot \ell ]$ with tolerance $\tau/2$ by estimating $\E_{\A^\Lambda_{|\chi}}[\frac{1}{2} (\phi(x,1) - \phi(x,-1)) \cdot \ell]$ with tolerance $(1-2\eta)\tau/4$ and then multiplying the result by $1/(1-2\eta)$. The estimation of $\E_{\A^\Lambda_{|\chi}}[\frac{1}{2} (\phi(x,1) - \phi(x,-1)) \cdot \ell]$ with tolerance $(1-2\eta)\tau/4$ can be done exactly as in Theorem \ref{th:simulate}.
\end{proof}
An immediate implication of Theorem \ref{th:simulate-noise-uncorr} is that one can simulate an active SQ algorithm $A$ using examples corrupted by noise $\Lambda$ as long as $\Lambda$ is $(\eta,(1-2\eta)\tau/4)$-uncorrelated with all $A$'s queries of tolerance $\tau$ for some fixed $\eta$.

Clearly, random classification noise of rate $\eta$ has function $\Lambda(x) = \eta$ for all $x \in X$. It is therefore $(\eta,0)$-uncorrelated with any query over any distribution. Another simple type of noise that is uncorrelated with most queries over most distributions is the one where noise function is chosen randomly so that for every point $x$ the noise rate $\Lambda(x)$ is chosen randomly and independently from some distribution with expectation $\eta$ (not necessarily the same for all points). For any fixed query and target distribution, the expected correlation is 0. If the probability mass of every single point of the domain is small enough compared to (the inverse of the logarithm of) the size of space of queries and target distributions then standard concentration inequalities will imply that the correlation will be small with high probability.

We would like to note that the noise models considered here are not directly comparable to the well-studied Tsybakov's and Massart's noise conditions~\cite{bousquet05survey}. However, it appears that from a computational point of view our noise model is significantly more benign than these conditions as they do not impose any structure on the noise and only limit the rate.

\subsection{Simple examples}
\paragraph{Thresholds:}
We show that a classic example of active learning a threshold function on an interval can be easily expressed using active SQs. For simplicity and without loss of generality we can assume that the interval is $[0,1]$ and the distribution is uniform over it. \footnote{As usual, we can bring the distribution to be close enough to this form using unlabeled samples or $O(b/\epsilon)$ target-independent queries, where $b$ is the number of bits needed to represent our examples.} Assume that we know that the threshold $\theta$ belongs to the interval $[a,b] \subseteq [0,1]$. We ask a query $\phi(x,\ell) = (\ell+1)/2$ with filter $\chi(x)$ which is the indicator function of the interval $[a,b]$ with tolerance $1/4$ and filter tolerance $b-a$. Let $v$ be the response to the query. By definition, $\E[\chi(x)] = b-a$ and therefore we have that $|v - \E[\phi(x,\ell) \cond x \in [a,b]]| \leq 1/4$. Note that,
$$\E[\phi(x,\ell) \cond x \in [a,b]] = (b-\theta)/(b-a)\ .$$ We can therefore conclude that $(b-\theta)/(b-a) \in [v-1/4,v+1/4]$ which means that $\theta \in [b - (v+1/4)(b-a), b - (v-1/4)(b-a)] \cap [a,b]$. Note that the length of this interval is at most $(b-a)/2$. This means that after at most $\log_{2}(1/\eps)+1$ iterations we will reach an interval $[a,b]$ of length at most $\eps$. In each iteration only constant $1/4$ tolerance is necessary and filter tolerance is never below $\eps$. A direct simulation of this algorithm can be done using $\log(1/\eps) \cdot \log(\log(1/\eps)/\delta)$ labeled examples and $\tilde{O}(1/\eps) \cdot \log(1/\delta)$ unlabeled samples.

\paragraph{Axis-aligned rectangles:}
Next we show that learning of thresholds can be used to obtain a simple algorithm for learning axis-aligned rectangles whose weight under the target distribution is not too small. Namely, we assume that the target function satisfies that $\E_D[f(x)] \geq \beta$. In the one dimensional case, we just need to learn an interval. After scaling the distribution to be uniform on $[0,1]$ we know that the target interval $[\theta_1,\theta_2]$ has length at least $\beta$. We first need to find a point inside that interval. To do this we consider the $2/\beta$ intervals $[(i-1)\beta/2,i\beta/2]$ for $1\leq i\leq 2/\beta$. At least one of these intervals in fully included in $[\theta_1,\theta_2]$. Hence using an active statistical query with query function $\phi(x,\ell) = (\ell+1)/2$ conditioned on being in interval $[(i-1)\beta/2,i\beta/2]$ for each $1\leq i\leq 2/\beta$ and with tolerance $1/4$ we are guaranteed to find an interval for which the answer is at least $3/4$. The midpoint of any interval for which the answer to the query is at least $3/4$ must be inside the target interval. Let the midpoint be $a$. We can now use two binary searches with accuracy $\eps/2$ to find the lower and upper endpoints of the target interval in the intervals $[0,a]$ and $[a,1]$, respectively. This will require $2/\beta + \log_{2}(2/\eps)$ active SQs of tolerance $1/4$. As usual, the $d$-dimensional axis-aligned rectangles can be reduced to $d$ interval learning problems with error $\eps/d$~\cite{KearnsVazirani:94}. This gives an active statistical algorithm using  $2d/\beta + \log_{2}(2d/\eps)$ active SQs of tolerance $1/4$ and filter tolerance $\geq \min\{\beta/2,\eps/2\}$.
%%Building on this one can easily learn unions of two rectangles that are sufficiently separated.
 %%Now assume are trying to learning a union of two rectangles R_1 union R2 that are sufficiently separated. Then this means they should be sufficiently
 %% separated in one axis.
 %%So, we all we have to do from a small number of  labeled examples, by applying alg A on all axeses we can determine on axis where there is separation, that %% is we find an hyperplane H perpendicular to that axes that separates R_1 from R_2.
 %% We then just split the problem into two problems, to the left of  H and to the right of H -- we can do this by using filtering. In each side we just have a %% rectangle problem to solve.

\paragraph{${\bf A^2}:$}
We now note that the general and well-studied $A^2$ algorithm of~\cite{BBL06} falls naturally into our framework.
At a high level, the  $A^2$ algorithm is an iterative, {\em disagreement-based}  active learning algorithm.  It maintains a set of surviving classifiers $C_i \subseteq C$,  and in each round the algorithm asks for the labels of a few random points that fall in the current region of disagreement of the surviving classifiers. Formally, the region of disagreement $\DIS(\C_i)$ of a set of classifiers $\C_i$ is the of set of instances $x$ such that for each  $x \in \DIS(\C_i)$ there exist  two classifiers $f,g \in \C_i$ that disagree about the label of $x$. Based on the queried labels, the algorithm then eliminates hypotheses that were still under consideration, but only if it is {\em statistically confident} (given the labels queried in the last round) that they are suboptimal. In essence,
in each round $A^2$ only needs to estimate the error rates (of hypotheses still under consideration) under the conditional distribution of being in the region of disagreement. The key point is that this can be easily done via active statistical queries. Note that while the number of active statistical queries needed to do this could be large, the number of labeled examples needed to simulate these queries is essentially the same as the number of labeled examples needed by the known $A^2$ analyses~\cite{Hanneke07,hanneke:survey}.
 %$A^2$ solves a sequence of standard passive learning tasks relative to the conditional distribution of being in the %disagreement region; furthermore, it only interacts with the conditional distribution given the region of %disagreement by only estimating error rates. Hence using a passive statistical algorithm for maintaining the %hypothesis space one obtains an active statistical implementation of $A^2$.
While in general the required computation of the disagreement region and manipulations of the hypothesis space cannot be done efficiently, efficient implementation is possible in a number of simple cases such as when the VC dimension of the concept class is a constant. It is not hard to see that in these cases the implementation can also be done using a statistical algorithm.

\section{Learning halfspaces with respect to log-concave distributions}
\label{sec:hs-logc}
%In this section we present a reduction from active learning to passive learning of homogeneous linear separators.
% Combining it with the SQ algorithm for learning halfspaces in the passive learning setting due to Dunagan and Vempala~\cite{DunaganVempala:04}, we obtain the first efficient noise-tolerant active learning of homogeneous halfspaces for any isotropic log-concave distribution. Our master reduction proceeds in rounds, and in each round $t$ we build a better approximation $w_t$ to the target function by using the~\cite{DunaganVempala:04} algorithm and active statistical queries that involve a real valued filter (that corresponds to a mixture of distributions that are conditional distributions within progressively refined bands around the previous approximations $w_i$).
%
%Though the procedure is a bit different from the classic margin based technique analysed in~\cite{BalcanBZ:07,BalcanLong:13} (that would translate into using a boolean filter corresponding to a conditional distribution around a band of the previous guess of the target), our proof technique builds on that of~\cite{BalcanBZ:07,BalcanLong:13}.
%In fact, one of the key point of this result is that it is relatively easy to harness the involved results developed for SQ framework to obtain new active statistical algorithms.

In this section we present a reduction from active learning to passive learning of homogeneous linear separators under log-concave distributions.
Combining it with the SQ algorithm for learning halfspaces in the passive learning setting due to Dunagan and Vempala~\cite{DunaganVempala:04}, we obtain the first efficient noise-tolerant active learning of homogeneous halfspaces for any isotropic log-concave distribution.

Our reduction proceeds in rounds;  in round $t$ we build a better
approximation $w_t$ to the target function by using the passive SQ
learning algorithm ~\cite{DunaganVempala:04} over a distribution $D_t$
that is a mixture of distributions in which each component is the
original distribution conditioned on being within a certain distance
from the hyperplane defined by previous approximations $w_i$. To
perform passive statistical queries relative to $D_t$ we use active
SQs with a corresponding real valued filter.
Our analysis builds on the analysis of the margin-based algorithms due to~\cite{BalcanBZ:07,BalcanLong:13}.
However, note that in the standard
margin-based analysis only points close to the current hypothesis
$w_t$ are queried in round $t$. As a result the analysis of our
algorithm is somewhat different from that in earlier work
~\cite{BalcanBZ:07,BalcanLong:13}.

\subsection{Preliminaries}
For a unit vector $v\in \R^d$ we denote by $h_v(x)$ the function defined by the homogenous hyperplane orthogonal to $v$, that is $h_v(x) = \sgn(\la v , x \ra)$. Let $\cH_d$ denote the concept class of all homogeneous halfspaces.

\begin{definition}
 A distribution over $\R^d$ is log-concave if $\log f( \cdot )$ is concave, where $f$ is its associated density function.
 It is isotropic if its mean is the origin and its covariance matrix  is the identity.
\end{definition}
Log-concave distributions form a broad class of distributions:
for example, the Gaussian, Logistic, Exponential,
and uniform distribution over any convex set are log-concave distributions.

Next, we state several simple properties of log-concave densities from \cite{LovaszVempala:07}.
\begin{lemma}
\label{lemma:logc-project}
There exists a constant $c_m$ such that for any isotropic log-concave
distribution $D$ on $\R^d$,  every unit vector $v$ and $a \in [0,1]$,
$$c_m a \leq \pr_D[x \cdot v \in [-a,a]]\leq 2a .$$
\end{lemma}
\begin{lemma}
\label{lemma:angle}
There exists a constant $c$ such that for any isotropic log-concave $D$ on $\R^d$ and any two unit vectors $u$ and $v$ in $\R^d$  we have
$ c \anglesep(u,v) \leq \E_D[h_u(x) \neq h_v(x)]$, where $\anglesep(u,v)$ denotes the angle between $u$ and $v$.
\end{lemma}

For our applications the key property of log-concave densities proved in ~\cite{BalcanLong:13} is given in the following lemma.
\begin{lemma}
\label{lemma:vectors-sophist}
For any constant $c_1 > 0$, there exists a constant $c_2 > 0$ such that the
following holds.
Let $u$ and $v$ be two unit vectors in $\R^d$, and assume that
$\theta(u,v) = \alpha < \pi/2$.  Assume that $D$
is isotropic log-concave in $\R^d$.  Then
\begin{equation}
\label{e:largemargin}
\pr_D [ h_u(x) \neq h_v(x)
\mbox{ and }
| v \cdot x| \geq c_2 \alpha]
 \leq c_1 \alpha.
\end{equation}
\end{lemma}

We now state the passive SQ algorithm for learning halfspaces which will be the basis of our active SQ algorithm.
\newcommand{\sqhs}{{\tt LearnHS}}
\begin{theorem}
\label{thm:dv-logc}
There exists a SQ algorithm \sqhs\ that learns $\cH_d$ to accuracy $1-\eps$ over any distribution $D_{|\chi}$, where $D$ is an isotropic log-concave distribution and $\chi:\R^d \rightarrow [0,1]$ is a filter function. Further \sqhs\ outputs a homogeneous halfspace, runs in time polynomial in $d$,$1/\eps$ and $\log(1/\lambda)$ and uses SQs of tolerance $ \geq 1/\poly(d,1/\eps,\log(1/\lambda))$, where $\lambda = \E_D[\chi(x)]$.
\end{theorem}
We use the Dunagan-Vempala algorithm for learning halfspaces to prove this algorithm \cite{DunaganVempala:04}. The bounds on the complexity of the algorithm follow easily from the properties of log-concave distributions. Further details of the analysis and related discussion appear in Appendix \ref{app-dv}.
\subsection{Active learning algorithm}
\newcommand{\asqlogc}{{\tt ActiveLearnHS-LogC}}
\begin{theorem}
\label{RCS-log}
There exists an active SQ algorithm \asqlogc\ (Algorithm \ref{fig:active-sq-half}) that for any isotropic log-concave distribution $D$ on $\R^d$, learns  $\cH_d$ over $D$ to accuracy $1-\eps$ in time $\poly(d,\log(1/\eps))$ and using active SQs of tolerance $\geq 1/\poly(d,\log(1/\eps))$ and filter tolerance $\Omega(\eps)$.
\end{theorem}
\begin{algorithm}
\caption{\asqlogc: Active SQ learning of homogeneous halfspaces over isotropic log-concave densities}
\begin{algorithmic}[1]
\STATE {\tt \%\% Constants $c$, $C_1$, $C_2$ and $C_3$ are determined by the analysis.}
\STATE Run \sqhs\ with error $C_2$ to obtain $w_0$.
\FOR{$k=1$ to $s =\lceil\log_2(1/(c\epsilon))\rceil$}
\STATE Let $b_{k-1} = C_1/2^{k-1}$
\STATE Let $\mu_k$ equal the indicator function of being within margin $b_{k-1}$ of $w_{k-1}$
\STATE Let $\chi_k = (\sum_{i\leq k} \mu_i)/k$
\STATE Run \sqhs\ over $D_k = D_{|\chi_k}$ with error $C_2/k$ by using active queries with filter $\chi_k$ and filter tolerance $C_3 \eps$ to obtain $w_k$
\ENDFOR
\RETURN $w_s$
\end{algorithmic}
\label{fig:active-sq-half}
\end{algorithm}
\begin{proof}
Let $c$ be the constant given by Lemma \ref{lemma:angle} and let $C_1$ be the constant $c_2$ given by Lemma \ref{lemma:vectors-sophist} when $c_1 = c/16$. Let $C_2 = c/(8C_1)$ and $C_3 = c_m \cdot C_2 \cdot c$. For every $k \leq s =\lceil\log_2(1/(c\epsilon))\rceil$ define $b_k = C_1/2^k$. Let $h_w$ denote the target halfspace and for any unit vector $v$ and distribution $D'$ we define $\err_{D'}(v) = \pr_{D'}[h_w(x) \neq h_v(x)]$.

We define $w_0, w_1,\ldots, w_s$ via the iterative process described in Algorithm~\ref{fig:active-sq-half}. Note that active SQs are used to allow us to execute \sqhs\ on $D_k$. That is a SQ $\phi$ of tolerance $\tau$ asked by \sqhs\ (relative to $D_k$) is replaced with an active SQ $(\chi_k,\phi)$ of tolerance $(C_3\eps,\tau)$. The response to the active SQ is a valid response to the query of \sqhs\ as long as $\E_D[\chi_k] \geq C_3\eps$. We will prove that this condition indeed holds later. We now prove by induction on $k$ that after $k \leq s$ iterations, we have that
every $\hat{w}$ such that $\err_{D_{|\mu_i}}(\hat{w}) \leq C_2$ for all $i \leq k$  satisfies
$\err_{D}(\hat{w}) \leq c/2^k$. In addition, $w_k$ satisfies this condition.

The case $k=0$ follows from the properties of \sqhs (without loss of generality $C_2 \leq c$).
Assume now that the claim is true for
$k-1$ ($k\geq 1$).  Let  $S_k^1=\{x:
|w_{k-1}\cdot x| \leq b_{k-1}\}$ and $S_k^2=\{x:
|w_{k-1}\cdot x| > b_{k-1}\}$. Note that $\mu_{k-1}$ is defined to be the indicator function of $S_k^1$.
By the inductive hypothesis we know that $\err_{D}(w_{k-1}) \leq c/2^{k-1}$.

Consider an arbitrary separator $\hat{w}$ that satisfies $\err_{D_{|\mu_i}}(\hat{w}) \leq C_2$ for all $i \leq k$.
By the inductive hypothesis, we know that $\err_{D}(\hat{w}) \leq c/2^{k-1}$.
By Lemma~\ref{lemma:angle} we have $\theta(\hat{w}, w) \leq 2^{-k+1}$ and $\theta(w_{k-1}, w) \leq 2^{-k+1}$. This implies $\theta(w_{k-1}, \hat{w}) \leq 2^{-k+2}$. By our choice of $C_1$ and Lemma~\ref{lemma:vectors-sophist}, we obtain:
\begin{eqnarray*}
&& \pr_D\left[\sgn(w_{k-1} \cdot x) \neq \sgn(\hat{w} \cdot x),\ x \in S_k^2\right]
  \leq  c2^{-k}/4  \\
&& \pr_D\left[\sgn(w_{k-1} \cdot x)  \neq \sgn(w \cdot x),\ x \in S_k^2\right]
  \leq  c2^{-k}/4.
\end{eqnarray*}

Therefore,
\begin{equation}
\label{eq:S2}
\pr_D\left[\sgn(\hat{w} \cdot x) \neq \sgn(w \cdot x),\ x \in S_k^2\right] \leq  c2^{-k}/2.
\end{equation}

By the inductive hypothesis, we also have:
$$\err_{D_{|\mu_k}}(\hat{w}) = \pr_D\left[\sgn(\hat{w} \cdot x) \neq \sgn(w \cdot x) \cond x \in S_k^1\right] \leq C_2 .$$
The set $S_k^1$ consists of points $x$ such that $x \cdot w_{k-1}$ fall into interval $[-b_{k-1},b_{k-1}]$. By Lemma \ref{lemma:logc-project}, this implies that $\pr_D[x \in S_k^1] \leq 2 b_{k-1}$ and therefore,
\alequ{
\label{eq:S1}
\nonumber
\pr_D\left[\sgn(\hat{w} \cdot x) \neq \sgn(w \cdot x),\ x \in S_k^1\right] &= \pr_D\left[\sgn(\hat{w} \cdot x) \neq \sgn(w \cdot x) \cond x \in S_k^1\right] \cdot \pr_D\left[x \in S_k^1\right] \\& \leq 2C_2 \cdot b_{k-1} = c2^{-k}/2.
}
Now by combining eq.~(\ref{eq:S2}) and eq.~(\ref{eq:S1}) we get that $\err_D(\hat{w}) \leq c/2^k$ as necessary to establish the first part of the inductive hypothesis.
By the properties of \sqhs, $\err_{D_k}(w_k) \leq C_2/k$. By the definition of $\chi_k$,
$$\err_{D_k}(w_k) = \fr{k} \sum_{i\leq k} \err_{D_{|\mu_i}}(w_k).
$$
%= \fr{k} \sum_{i\leq k} \pr_D[\sgn(w_k \cdot x) \neq \sgn(w \cdot x) \cond x \in S_i^1]\ .$$
This implies that for every $i \leq k$, $\err_{D_{|\mu_i}}(w_k) \leq C_2 ,$ establishing the second part of the inductive hypothesis.

Inductive hypothesis immediately implies that $\err_D(w_s) \leq \eps$. Therefore to finish the proof we only need to establish the bound on running time and query complexity of the algorithm. To establish the lower bound on filter tolerance we observe that by Lemma \ref{lemma:logc-project}, for every $k \leq s$, $$\E_D[\mu_k] \geq c_m \cdot b_{k-1} = c_m \cdot C_2 / 2^{k-1} \geq c_m \cdot C_2 \cdot c \cdot \eps .$$ This implies that for every $k\leq s$, $$\E_D[\chi_k]  = \fr{k}\sum_{i\leq k} \E_D[\mu_i] = \Omega(\eps) .$$
Each execution of \sqhs\ is with error $C_2/k = \Omega(1/\log(1/\eps))$ and there are at most $O(\log(1/\eps))$ such executions.
Now by Theorem \ref{thm:dv-logc} this implies that the total running time, number of queries and the inverse of query tolerance are upper-bounded by a polynomial in $d$ and $\log(1/\eps)$.
\end{proof}
We remark that, as usual, we can first bring the distribution to an isotropic position by using target independent queries to estimate the mean and the covariance matrix of the distribution~\cite{LovaszVempala:07}. Therefore our algorithm can be used to learn halfspaces over general log-concave densities as long as the target halfspace passes through the mean of the density. %Furthermore the results here also extend to $\beta$-log-concave distributions~\cite{BalcanLong:13} for sufficiently small constant $\beta$. This is a broader class of distributions introduced by Applegate and Kannan~\cite{ak:91} which contains mixtures of (not too separated) log-concave distributions.

We can now apply Theorem \ref{th:simulate-noise} (or more generally Theorem \ref{th:simulate-noise-uncorr}) to obtain an efficient active learning algorithm for homogeneous halfspaces over log-concave densities in the presence of random classification noise of known rate. Further since our algorithm relies on \sqhs\ which can also be simulated when the noise rate is unknown (see Remark \ref{rem:unknown-noise}) we obtain an active algorithm which does not require the knowledge of the noise rate.
\begin{corollary}
\label{cor:learn-hs-rcn-noise-log-concave}
There exists a polynomial-time active learning algorithm that for any $\eta \in [0,1/2)$, learns $\cH_d$ over any log-concave distributions with random classification noise of rate $\eta$ to error $\eps$ using
 $\poly(d, \log(1/\eps), 1/(1-2\eta))$ labeled examples and a polynomial number of unlabeled samples.
\end{corollary}

\section{Learning halfspaces over the uniform distribution}
\label{sec:uniform}
The algorithm presented in Section \ref{sec:hs-logc} relies on the relatively involved and computationally costly algorithm of Dunagan and Vempala \cite{DunaganVempala:04} for learning halfspaces over general distributions. Similarly, other active learning algorithms for halfspaces often rely on the computationally costly linear program solving \cite{BalcanBZ:07,BalcanLong:13}. For the special case of the uniform distribution on the unit sphere we now give a substantially simpler and more efficient algorithm in terms of both sample and computational complexity. This setting was studied in \cite{BalcanBZ:07,DasguptaKM:09}.

We remark that the  uniform distribution over the unit sphere is not log-concave and therefore, in general, an algorithm for the isotropic log-concave case might not imply an algorithm for the uniform distribution over the unit sphere. However a more careful look at the known active algorithms for the isotropic log-concave case \cite{BalcanBZ:07,BalcanLong:13} and at the algorithms in this work shows that minimization of error is performed over homogeneous halfspaces. For any homogeneous halfspace $h_v$,  any $x \in \R^d$ and $\alpha > 0$, $h_v(x) = h_v(\alpha x)$. This implies that for algorithms optimizing the error over homogenous halfspaces any two spherically symmetric distributions are equivalent. In particular, the uniform distribution over the sphere is equivalent to the uniform distribution over the unit ball -- an isotropic and log-concave distribution.

For a dimension $d$ let $X = S_{d-1}$ or the unit sphere in $d$ dimensions. Let $U_d$ denote the uniform distribution over $S_{d-1}$. Unless specified otherwise, in this section all probabilities and expectations are relative to $U_d$. We would also like to mention explicitly the following trivial lemma relating the accuracy of an estimate of $f(\alpha)$ to the accuracy of an estimate of $\alpha$.
\begin{lemma}
\label{lem:mvt}
Let $f:\R\rightarrow\R$ be a differentiable function and let $\alpha$ and $\tilde{\alpha}$ be any values in some interval $[a,b]$. Then $$|f(\alpha) - f(\tilde{\alpha})| \leq |\alpha - \tilde{\alpha}| \cdot \sup_{\beta\in [a,b]} |f'(\beta)|\ .$$
\end{lemma}
The lemma follows directly from the mean value theorem. Also note that given an estimate $\tilde{\alpha}$ for $\alpha \in [a,b]$ we can always assume that $\tilde{\alpha} \in [a,b]$ since otherwise $\tilde{\alpha}$ can be replaced with the closest point in $[a,b]$ which will be at least as close to $\alpha$ as $\tilde{\alpha}$.

We start with an outline of a non-active and simpler version of the algorithm that demonstrates one of the ideas of the active SQ algorithm. To the best of our knowledge the algorithm we present is also the simplest and most efficient (passive) SQ algorithm for the problem. A less efficient algorithm is given in \cite{KanadeVV:10}.
\subsection{Learning using (passive) SQs}
Let $w$ denote the normal vector of the target hyperplane and let $v$ be any unit vector.
Instead of arguing about the disagreement between $h_w$ and $h_v$ directly we will use the (Euclidean) distance between $v$ and $w$ as a proxy for disagreement. It is easy to see that, up to a small constant factor, this distance behaves like disagreement.
\begin{lemma}
\label{lem:error-2-dist}
For any unit vectors $v$ and $w$,
\begin{enumerate}
\item \label{it:error-from-dist} Error is upper bounded by half the distance: $\pr[h_v(x) \neq h_w(x)] \leq \|w-v\|/2$;
\item \label{it:estim-dist} To estimate distance it is sufficient to estimate error: for every value $\alpha \in [0,1]$, $$\left|\|w-v\| - 2 \sin(\pi \alpha/2)\right| \leq \pi \left| \pr[h_v(x) \neq h_w(x)] - \alpha \right|\ .$$
%\item Decrease in distance implies decrease in error: if $\|w-v\| > \|w-v'\|$ then $$\pr[h_v(x) \neq h_w(x)] - \pr[h_{v'}(x) \neq h_w(x)] \geq \pi (\|w-v\| - \|w-v'\|) .$$
\end{enumerate}
\end{lemma}
\begin{proof}
The angle between $v$ and $w$ equals $\gamma = \pi \pr[h_v(x) \neq h_w(x)]$. Hence
$$\|w-v\| = 2 \sin(\pi \gamma/2) = 2 \sin(\pi \pr[h_v(x) \neq h_w(x)]/2)$$ and $\pr[h_v(x) \neq h_w(x)] = 2\arcsin(\|w-v\|/2)/\pi.$ The first claim follows by observing that $\frac{2 \arcsin(x/2)}{\pi x}$ is a monotone function in $[0,2]$ and equals $1/2$ when $x=2$.

The derivative of $2 \sin(\pi x/2)$ equals at most $\pi$ in absolute value and therefore the second claim follows from Lemma \ref{lem:mvt}. %In addition the derivative is non-negative whenever $x \leq 2$ giving the third claim again by mean value theorem.
\end{proof}

The main idea of our algorithm is as follows. Given a current hypothesis represented by its normal vector $v$, we estimate the distance from the target vector $w$ to $v$ perturbed in the direction of each the $d$ basis vectors. By combining the distance measurements in these directions we can find an estimate of $w$. Specifically, let $\{x^1,x^2,\ldots,x^{d}\}$ be the unit vectors of the standard basis. Let $v^i = (v + \beta x^i)/\|v + \beta x^i\|$ for some $\beta \in (0,1/2]$. Then the distance from $w$ to $v^i$ can be used to (approximately) find $w \cdot x^i$. Namely, we rely on the following  simple lemma.
\begin{lemma}
\label{lem:reconstruct-coordinate}
Let $u,v$ and $w$ be any unit vectors and $\beta \in (0,1/2]$. Then for $v' = (v + \beta u)/\|v+ \beta u\|$ it holds that $$\la u, w \ra = \frac{\|v+ \beta u\| (2-\|v'-w\|^2) - 2 + \|v-w\|^2}{2\beta} \ .$$
\end{lemma}
\begin{proof}
By definition, $\la v' , w \ra =  (\la v , w \ra + \beta \la u , w \ra)/\|v+ \beta u\|$ and therefore
\equ{\la u , w \ra = \frac{\|v+ \beta u\| \la v' , w\ra - \la v , w \ra }{\beta}\ .\label{eq:get-coord}} For every pair of unit vectors $u$ and $u'$, $\la u , u' \ra = \frac{\|u\|^2 + \|u'\|^2 - \|u-u'\|^2}{2} = 1-\|u-u'\|^2/2$ and therefore, we get that $\la v, w \ra = 1 - \|w-v\|^2/2$ and $\la v', w \ra = 1-\|w-v'\|^2/2$. Substituting those into eq.~(\ref{eq:get-coord}) gives the claim.
\end{proof}
Using approximate values of $\la x^i , w \ra$ for all $i \in [d]$ one can easily approximate $w$. Our (non-active) SQ learning algorithm provides a simple example of how such reconstruction can be used for learning.
\newcommand{\hsu}{{\tt LearnHS-U}}
\begin{theorem}
\label{th:learnHSinSQ}
There exists a polynomial time SQ algorithm \hsu\ that learns $\cH_d$ over $U_d$ using $d+1$ statistical queries each of tolerance $\Omega(\eps/\sqrt{d})$.
\end{theorem}
\begin{proof}
%such that $\Delta = \|v-w\| \geq 1$
Let $v$ be any unit vector (\eg $x^1$) and for $i \in [d]$, define $v^i = (v + x^i/2)/\|v+ x^i/2\|$ (that is $\beta = 1/2$).
Let $h_w$ denote the unknown target halfspace. For every $v^i$, we ask a statistical query with tolerance $\eps/(10 \cdot \pi \sqrt{d})$ to obtain $\alpha_i$ such that $|\pr[h_{v^i} \neq h_w] - \alpha_i| \leq \eps/(10\cdot \pi \sqrt{d})$ and similarly get $\alpha$ such that $|\pr[h_{v} \neq h_w] - \alpha| \leq \eps/(10 \cdot \pi \sqrt{d})$.

We define $\gamma = 1-(2\sin(\pi \alpha/2))^2/2$ and for every $i\in [d]$,
$$\gamma_i = \|v+ x^i/2\| (2- (2 \sin(\pi \alpha_i/2))^2) - 2 \gamma\ .$$

By Lemma \ref{lem:error-2-dist}(\ref{it:estim-dist}), we get that $|\|v - w\| - 2 \sin(\pi \alpha/2)| \leq \eps/(10 \sqrt{d})$.
Clearly $\|v - w\| \leq 2$ and therefore, by Lemma \ref{lem:mvt}, $$|\gamma - \la v ,w \ra| = |(2 \sin(\pi \alpha/2))^2 - \|v - w\|^2|/2 \leq
4 \cdot |(2 \sin(\pi \alpha/2)) - \|v - w\||/2 \leq  \eps/(5 \sqrt{d}) .$$ Note that $\|v+ x^i/2\| \leq 3/2$ and therefore, by Lemmas \ref{lem:reconstruct-coordinate} and \ref{lem:mvt} $$\left|\gamma_i - \la x^i ,w \ra\right| \leq \frac{3}{2} \cdot 4 \cdot \eps/(10\sqrt{d}) +  2\eps/(5 \sqrt{d}) = \eps/\sqrt{d}\ .$$

Now let $$w' = \sum_{i\in [d]} \gamma_i x^{i}\ .$$ Parseval's identity implies that
$$\|w-w'\|^2 = \sum_{i \in [d]} |\gamma_i - \la x^i ,w \ra|^2 \leq d \cdot \frac{\eps^2}{d} = \eps^2\ . $$
Let $w^* = w'/\|w'\|$. Clearly, $\|w^*-w'\| \leq \|w-w'\| \leq \eps$ and therefore, by triangle inequality, $\|w^*-w\| \leq 2\eps$. By Lemma \ref{lem:error-2-dist}(\ref{it:error-from-dist}) this implies that $\pr[h_{w^*} \neq h_w] \leq \eps$.
It is easy to see that this algorithm uses $d+1$ statistical queries of tolerance $\Omega(\eps/\sqrt{d})$ and runs in time linear in $d$.
\end{proof}
\begin{remark}
It is not hard to see that an even simpler way to find each of the coordinates of $w$ is by measuring the error of each of the standard basis vectors themselves and using the fact that $\la w , x^i \ra = \cos(\pi \cdot \pr[h_{x^i} \neq h_w])$. The variant we presented in Theorem \ref{th:learnHSinSQ} is more useful as a warm-up for the analysis of the active version of the algorithm.
\end{remark}

\subsection{Active Learning of Halfspaces over $U_d$}
Our active SQ learning algorithm is based on two main ideas. First, as in Theorem \ref{th:learnHSinSQ}, we rely on measuring the error of hypotheses which are perturbations of the current hypothesis in the direction of each of the basis vectors. We then combine the measurements to obtain a new hypothesis. Second, as in previous active learning algorithms for the problem \cite{DasguptaKM:09,BalcanBZ:07}, we only use labeled examples which are within a certain margin of the current hypothesis. %Here the margin of point $x$ relative to the hyperplane defined by $v$ refers to the geometrical distance from $x$ to the hyperplane $v$ or $|\la v, x \ra|$.
The margin we use to filter the examples is a function of the current error rate. It is implicit in previous work \cite{DasguptaKM:09,BalcanBZ:07,BalcanLong:13} that for an appropriate choice of margin, a constant fraction of the error region is within the margin while the total probability of a point being within the margin is linear in the error of the current hypothesis. Together these conditions allow approximating the error of a hypothesis using tolerance that has no dependence on $\eps$.

We start by computing the error of a hypothesis $v$ whose distance from the target $w$ is $\Delta$ conditioned on being within margin $\gamma$ of $v$. Let $A_{d-1}$ denote the surface area of $S_{d-1}$. First, the surface area within margin $\gamma$ of any homogenous halfspace $v$ is
\equ{2 \int_0^\gamma A_{d-2} (1-r^2)^{(d-2)/2} \cdot \frac{1}{\sqrt{1-r^2}} dr = 2 \cdot A_{d-2} \int_0^\gamma  (1-r^2)^{(d-3)/2} dr .\label{eq:hs-band}}
We now observe that for any $v$ and $w$ such that $\|v-w\| = \Delta$, $\pr [h_v(x) \neq h_w(x) \cond |\la v , x \ra | \leq \gamma]$ is a function that depends only on $\Delta$ and $\gamma$.
\begin{lemma}
\label{lem:define-cp}
For any $v,w\in S_{d-1}$ such that $\|v-w\| = \Delta \leq \sqrt{2}$ and $\gamma > 0$, $$\pr [h_v(x) \neq h_w(x) \cond |\la v , x \ra | \leq \gamma] = \frac{A_{d-3} \int_0^\gamma (1-r^2)^{(d-3)/2}\int_{\frac{r \cdot \sqrt{2-\Delta^2}}{\Delta \cdot \sqrt{1-r^2}}}^1  (1-s^2)^{(d-4)/2}ds\cdot dr }{A_{d-2} \int_0^\gamma  (1-r^2)^{(d-3)/2} dr}. $$ We denote the probability by $\cp_d(\gamma,\Delta)$.
\end{lemma}
The proof of this lemma can be found in Appendix \ref{app-proofs-uniform}.

The analysis of our algorithm is based on relating the effect that the change in distance of a hypothesis has on the conditional probability of error. This effect can be easily expressed using the derivative of the conditional probability as a function of the distance. Specifically we prove the following lower bound on the derivative (the proof can be found in Appendix \ref{app-proofs-uniform}).

\begin{lemma}
\label{lem:cp-deriv}
For $\Delta \leq \sqrt{2}$, any $d\geq 4$, and $\gamma \geq \Delta/(2\sqrt{d})$, $\partial_\Delta \cp_d(\gamma,\Delta) \geq 1/(56\gamma \cdot \sqrt{d})$.
\end{lemma}

An important corollary of Lemma \ref{lem:cp-deriv} is that given a hypothesis $h_v$ and $\Delta$, such that $\|v-w\| \leq \Delta$ we can estimate $\|v-w\|$ to accuracy $\rho$ using an active statistical query with query tolerance $\Omega(\rho/\Delta)$. Specifically:
\begin{lemma}
\label{lem:cond-distance-est}
Let $h_w$ be the target hypothesis. There is an algorithm {\tt MeasureDistance}$(v,\rho,\Delta')$ that given a unit vector $v$ and $\rho > 0$ and $\Delta'$ such that $\|v - w\| \leq \Delta'$, outputs a value $\tilde{\Delta}$ satisfying $|\|v - w\| - \tilde{\Delta}| \leq \rho$. The algorithm asks a single active SQ with filter tolerance $\tau_0 = \Omega(\Delta')$ and query tolerance of $\tau = \Omega(\rho/\Delta')$ and runs in time $\poly(\log(1/(\Delta\rho)))$.
\end{lemma}
\begin{proof}
Let $\gamma = \Delta'/(2\sqrt{d})$. This implies that $\gamma \geq \|v-w\|/(2\sqrt{d})$. Lemma \ref{lem:cp-deriv} together with the mean value theorem imply that if $\Delta_1,\Delta_2 \leq \Delta'$ and $\Delta_1 - \Delta_2 \geq \rho$ then
for some $\hat{\Delta} \in [\Delta_1,\Delta_2]$,
 \begin{eqnarray}\cp_d(\gamma,\Delta_1) - \cp_d(\gamma,\Delta_2) = \rho \cdot \partial_\Delta \cp_d(\gamma,\hat{\Delta}) \geq \rho/(56 \gamma \sqrt{d}) = \rho/(28 \Delta')\ .
\label{eq-lb}
 \end{eqnarray}
 This implies that in order to estimate $\|v-w\|$ to within tolerance $\rho$ it is sufficient to estimate $\cp_d(\gamma,\|v-w\|)$ to within $\rho/(28 \Delta')$.  To see this note that an estimate of $\cp_d(\gamma,\|v-w\|)$ within $\rho/(28 \Delta')$ is a value
$\mu$ such that $|\mu - \cp_d(\gamma,\|v-w\|)| \leq \rho/(28\Delta')$. Let $\tilde{\Delta}$ be such that $\cp_d(\gamma,\tilde{\Delta})=\mu$.  Note that Lemma \ref{lem:define-cp} does not give an explicit mapping from $\cp_d(\gamma,\Delta)$ to $\Delta$. But $\cp_d(\gamma,\Delta)$ is a monotone function of $\Delta$ and can be computed efficiently given $\Delta$. Therefore we can efficiently invert $\cp_d(\gamma,\Delta)$ using a simple binary search.
This computation will give us a value $\tilde{\Delta}$ such that $|\cp_d(\gamma,\tilde{\Delta}) - \cp_d(\gamma,\|v-w\|)| \leq \rho/(28\Delta')$. Using this together with eq.~(\ref{eq-lb}), we obtain that $| \tilde{\Delta} - \|v-w\|| \leq \rho$.

Let $``|\la v , x \ra |\leq \gamma"$ denote the function (of $x$) that outputs 1 when the condition is satisfied and $0$ otherwise. By definition, $\cp_d(\gamma,\|v-w\|)$ can be estimated to within $\rho/(28 \Delta')$ using an active SQ $(``| \la v , x \ra |\leq \gamma"; h_v(x) \cdot \ell)$ with query tolerance of $\tau = \rho/(28 \Delta')$ and filter tolerance of $\tau_0 = \Delta'/8 \leq \pr_{U_d}[| \la v , x \ra |\leq \gamma]$ (for example see \cite{DasguptaKM:09}).
\end{proof}

We can now use the estimates of distance of a vector to $w$ (the normal vector of the target hyperplane) and Lemma \ref{lem:reconstruct-coordinate} to obtain a vector which is close to $w$. We perform this iteratively until we obtain a vector $v$ giving a hypothesis with error of at most $\eps$.

\newcommand{\actHSuni}{{\tt ActiveLearnHS-U}}
\begin{theorem}
\label{th:learn-hs-asq}
There exists an active statistical algorithm \actHSuni\ that learns $\cH_d$ over $U_d$ to accuracy $1-\eps$, uses $(d+1)\log(1/\eps)$ active SQs with tolerance of $\Omega(1/\sqrt{d})$ and filter tolerance of $\Omega(1/\eps)$ and runs in time $d \cdot \poly(\log{(d/\eps)})$.
\end{theorem}

\begin{algorithm}
\caption{\actHSuni: \textit{Active} SQ Learning of Homogeneous Halfspaces over the Uniform Distribution}
\begin{algorithmic}[1]
\STATE Run \hsu\ with parameter error $\eps'= 1/(2\pi)$ to obtain $u^1$
\FOR{$t=1$ to $\lceil \log(1/\epsilon)\rceil-2$}
\STATE Set $\alpha=${\tt MeasureDistance}$(u^{t}, \frac{1}{8 \cdot 2^t\sqrt{d}}, 2^{-t})$
\FOR{$i=1$ to $d$}
\STATE Set $v^i = (v + 2^{-t} x^i)/\|v+ 2^{-t} x^i\|$
\STATE Set $\alpha_i=${\tt MeasureDistance}$(v^{i},\frac{1}{24 \cdot 2^t\sqrt{d}},2^{-t+1})$
\STATE Set $\gamma_i = 2^{t-1}(\|v+ 2^{-t} x^i\| (2-\alpha_i^2) - 2 + \alpha^2)\ $
\ENDFOR
\STATE Set $v' = \sum_{i\in [d]} \gamma_i x^{i}\ $
\STATE Set $u^{t+1} = v'/\|v'\|$
\ENDFOR
\RETURN $u^{\lceil \log(1/\epsilon)\rceil-1}$.
\end{algorithmic}
\label{alg:act-hs-uni}
\end{algorithm}

\begin{proof}
Our algorithm works by finding a vector $v$ that is at distance of at most $2 \eps$ from the normal vector of the target hyperplane which be denote by $w$. We do this via an iterative process such that at step $t$ we construct a vector $u^t$, satisfying $\|w - u^t\| \leq 2^{-t}$.
In step 1 we construct a vector $u^1$ such that $\| u^1 - w \| \leq 1/2$ by using the (non-active) algorithm \hsu\ (Theorem \ref{th:learnHSinSQ}) with error parameter of $\eps' = 1/(2\pi)$. By Lemma \ref{lem:error-2-dist}(\ref{it:estim-dist}) we get that $\pr[h_w \neq h_{u^1}] \leq 1/(2\pi)$ implies that $\| u^1 - w \| \leq 1/2$.

Now given a vector $v=u^t$ such that $\|w - v\| \leq 2^{-t}$ we construct a unit vector $v^*$ such that $\|w - v^*\| \leq 2^{-t-1}$ and set $u^{t+1} = v^*$. Clearly, for $w^* = u^{\lceil \log{(1/\eps)} \rceil -1}$ we will get $\|w-w^*\| \leq 2\eps$ and hence, by Lemma \ref{lem:error-2-dist}, $\pr[h_w \neq h_{w^*}] \leq \eps$.

Let $\Delta' = 2^{-t}$, $\beta = 2^{-t}$ and define $v^i = (v + \beta x^i)/\|v+ \beta x^i\|$.
For every $v^i$, we know that $\|v-v^i\| \leq \beta = \Delta'$, this means $\|w - v^i\| \leq 2\Delta'$. We use
{\tt MeasureDistance} (Lemma \ref{lem:cond-distance-est}) for $v$, distance bound $\Delta'$ and accuracy parameter $\rho = \Delta'/(8\sqrt{d})$ to obtain $\alpha$ such that $|\|v - w\| - \alpha| \leq \Delta'/(8\sqrt{d})$. Similarly,
for each $i \in [d]$ we use {\tt MeasureDistance} for $v^i$, distance bound $2\Delta'$ and parameter $\rho = \Delta'/(24\sqrt{d})$ to obtain $\alpha_i$ such that $|\|v^i - w\| - \alpha_i| \leq \Delta'/(24\sqrt{d})$.

For $i\in [d]$, we define
$$\gamma_i = \frac{\|v+ \beta x^i\| (2-\alpha_i^2) - 2 + \alpha^2}{2\beta}\ .$$

We view $\gamma_i$ as a function of $\alpha$ and $\alpha_i$ and observe that for $\alpha \in [0, \Delta']$, $$\left|\partial_\alpha \gamma_i\right| = \left|\frac{\alpha}{\beta}\right| \leq 1$$ and for $\alpha_i \in [0,2\Delta']$, $$\left|\partial_{\alpha_i} \gamma_i\right| = \left|\frac{-\|v+ \beta x^i\| \alpha_i}{\beta}\right| \leq 2 \|v+ \beta x^i\| \leq 3\ . $$

Therefore, by Lemma \ref{lem:reconstruct-coordinate} and Lemma \ref{lem:mvt}, $$\left|\gamma_i - \la x^i ,w \ra\right| \leq 3 \cdot \Delta'/(24\sqrt{d}) + \Delta'/(8\sqrt{d}) = \Delta'/(4\sqrt{d})\ .$$

Now let $$v' = \sum_{i\in [d]} \gamma_i x^{i}\ .$$ Parseval's identity implies that
$$\|w-v'\|^2 = \sum_{i \in [d]} |\gamma_i - \la x^i ,w \ra|^2 \leq d \cdot \Delta'^2/(16d) = \Delta'^2/16\ . $$
Let $v^* = v'/\|v'\|$. Clearly, $\|v^*-v'\| \leq \|w-v'\| \leq \Delta'/4$ and therefore, by triangle inequality, $\|v^*-w\| \leq \Delta'/2 = 2^{-t-1}$.

All that is left to prove are the claimed bounds on active SQs used in this algorithm and its running time.
First note that each step uses $d+1$ active SQs and there are at most $\log(1/\eps)$ steps. By Lemma \ref{lem:cond-distance-est} the tolerance of each query at step $t$ is $\Omega((\Delta'/\sqrt{d})/\Delta') = \Omega(1/\sqrt{d})$ and filter tolerance is $\Omega(2^{-t}) = \Omega(\eps)$. Lemma \ref{lem:cond-distance-est} together with the bound on the number of stages also imply the claimed running time bound.
\end{proof}

An immediate corollary of Theorems \ref{th:learn-hs-asq} and \ref{th:simulate-noise} is an active learning algorithm for $\cH_d$ that works in the presence of random classification noise.
\begin{corollary}
\label{cor:learn-hs-asq-noise}
There exists a polynomial-time active learning algorithm that given any $\eta \in [0,1/2)$, learns $\cH_d$ over $U_d$ with random classification noise of rate $\eta$ to error $\eps$ using $O((1-2\eta)^{-2} \cdot d^2 \log(1/\eps) \log(d \log(1/\eps)))$ labeled examples and $O((1-2\eta)^{-2} \cdot d^2 \cdot \eps^{-1} \log(1/\eps) \log(d \log(1/\eps)))$ unlabeled samples.
\end{corollary}

\subsection{Learning with unknown noise rate}
One limitation of Corr.~\ref{cor:learn-hs-asq-noise} is that the simulation requires knowing the noise rate $\eta$.
We show that this limitation can be overcome by giving a procedure that approximately finds the noise rate which can then be used in the simulation of \actHSuni. The idea for estimating the noise rate is to measure the agreement rate of random halfspaces with the target halfspace. The agreement rate of a fixed halfspace is a linear function of the noise rate and therefore by comparing the distribution of agreement rates in the presence of noise to the distribution of agreement rates in the noiseless case we can factor out the noise rate.
\begin{lemma}
\label{lem:calculate-noise-rate}
There is an algorithm $\B$ that for every unit vector $w$ and values $\eta \in [0,1/2)$, $\tau,\delta \in(0,1)$, given $\tau,\delta$ and access to random examples from distribution $\A^\eta = (U_d, (1-2\eta) h_w)$ will, with probability at least $1-\delta$, output a value $\eta'$ such that $\frac{1-2\eta}{1-2\eta'} \in [1-\tau,1+\tau]$. Further, $\B$ runs in time polynomial in $d,1/\tau,1/(1-2\eta)$ and $\log(1/\delta)$ and uses $O(d \tau^{-2} (1-2\eta)^{-2} \cdot \log (d/((1-2\eta)\tau \delta)))$ random examples.
\end{lemma}
\begin{proof}
We consider the expected correlation of a randomly chosen halfspace with some fixed unknown halfspace $h_u$. Namely let
$$\nu = \E_{v \sim U_d}[|\E_{x \sim U_d}[h_u(x) \cdot h_v(x)]|] .$$ Spherical symmetry implies that $\nu$ does not depend on $u$. First note that $$\E_{x \sim U_d}[h_u(x) \cdot h_v(x)] = 2\arcsin(\la u , v \ra)/\pi \geq 2 \la u , v \ra/\pi\ .$$
A well-known fact is that for a randomly and uniformly chosen unit vector $v$, with probability at least $1/8$,  $\la u , v \ra \geq 1/\sqrt{d}$. This implies that with probability at least $1/8$, $\E_{x \sim U_d}[h_u(x) \cdot h_v(x)] \geq 2/(\pi \sqrt{d})$ and hence $\nu \geq c/ \sqrt{d}$ for some fixed constant $c$. Henceforth we can assume that $\nu$ is known exactly (as it is easy to estimate the necessary integral with the accuracy sufficient for our use).

At the same time we know that $$\E_{\A^\eta}[ \ell \cdot h_v(x)] = (1-2\eta) \E_{U_d}[h_w(x) \cdot h_v(x)]\ .$$ This means that
$$\nu_\eta = \E_{v \sim U_d}[|\E_{(x,\ell) \sim \A^\eta}[h_w(x) \cdot h_v(x)]|] = (1-2\eta)\nu\ .$$
Therefore in order to estimate $1-2\eta$ we estimate $\nu_\eta$ given samples from $\A^\eta$. To estimate $\nu_\eta$ we draw a set of random unit vectors $V$ and a set $S$ of random examples from $\A$. For each $v \in V$ we estimate $|\E_{(x,\ell) \sim \A^\eta}[h_w(x) \cdot h_v(x)]|$ using the random examples in $S$ and let $\alpha_v$ denote the corresponding estimate. The average of $\alpha_v$'s is an estimate of $\nu_\eta$.
Chernoff-Hoeffding bounds imply that for $t_1(\theta,\delta') = O(\theta^{-2} \log (1/\delta'))$ and $t_2(\theta,\delta') = O(\theta^{-2} \log (t_1(\theta,\delta')/\delta'))$, the estimation procedure above with $|V| = t_1(\theta,\delta')$ and $S = t_2(\theta,\delta')$, will with probability $1-\delta'$ return an estimate of $\nu_\eta$ within $\theta$.

We first find a good lower bound on $1-2\eta$ via a simple guess, estimate and double process. For $i=1,2,3,\ldots$ we estimate $\nu_\eta$ with tolerance $\nu \cdot 2^{-i}$ and confidence $\delta/2^{-i-1}$ until we get an estimate that equals at least $2 \cdot \nu \cdot 2^{-i}$. Let $i_\eta$ denote the first step at which this condition was satisfied. We claim that with probability at least $1-\delta/2$, $$2^{-i_\eta} \leq (1-2\eta) \leq 3 \cdot 2^{-i_\eta+1}.$$

First we know that the estimates are successful for every $i$ with probability at least $1-\delta/2$. The stopping condition implies that $\nu_\eta \geq \nu \cdot 2^{-i_\eta}$ and in particular, $(1-2\eta) \geq 2^{-i_\eta}$. Now to prove that $(1-2\eta) \leq 3 \cdot 2^{-i_\eta+1}$ we show that $i_\eta \leq \lceil \log{(3/(1-2\eta))} \rceil$. This is true since for $k= \lceil \log{(3/(1-2\eta))} \rceil$ we get that $(1-2\eta) \geq 3 \cdot 2^{-k}$ and hence $\nu_\eta = (1-2\eta)\nu \geq 3 \nu 2^{-k}$. Therefore an estimate of $\nu_\eta$ with tolerance $\nu \cdot 2^{-k}$ must be at least $2\nu \cdot 2^{-k}$. This means that $i_\eta \leq k$.

Given a lower bound of $2^{-i_\eta}$, we estimate $\nu_\eta$ to accuracy $\nu \tau 2^{-i_\eta}/2 \leq (1-2\eta) \nu \tau/2$ with confidence $1-\delta/2$ and let $\nu'_\eta$ denote the estimate. We set $1-2\eta' = \nu'/\nu$. We first note that $|(1-2\eta') - (1-2\eta)| \leq 2^{-i_\eta} \tau/2 \leq (1-2\eta) \tau/2$ and therefore $$\frac{1-2\eta}{1-2\eta'} \in \left[\frac{1}{1+\frac{\tau}{2}},\frac{1}{1-\frac{\tau}{2}}\right] \subseteq [1-\tau,1+\tau]\ .$$

Using the fact that $\nu \geq c/ \sqrt{d}$ and $i_\eta \leq \lceil \log{(3/(1-2\eta))} \rceil$ we can conclude that
the first step of the estimation procedure requires $O(d (1-2\eta)^{-2} \cdot \log \frac{d}{(1-2\eta)\delta})$ examples and the second step requires $O(d \tau^{-2} (1-2\eta)^{-2} \cdot \log \frac{d}{(1-2\eta)\delta\tau})$ examples. The straightforward implementation has running time of $O(d^3 \tau^{-4} (1-2\eta)^{-4} \cdot \log \frac{d}{(1-2\eta)\delta\tau})$.
\end{proof}

Note that by Theorem \ref{th:learn-hs-asq} our algorithm for learning halfspaces uses $\tau = \Omega(1/\sqrt{d})$. We can apply Lemma \ref{lem:calculate-noise-rate} together with Remark \ref{rem:unknown-noise} and Corollary \ref{cor:learn-hs-asq-noise} to obtain a version of the algorithm that does not require the knowledge of $\eta$.
\begin{corollary}
\label{cor:learn-hs-asq-noise-unk}
There exists a polynomial-time active learning algorithm that for any $\eta \in [0,1/2)$, learns $\cH_d$ over $U_d$ with random classification noise of rate $\eta$ to error $\eps$ using $$O\left((1-2\eta)^{-2} \cdot d^2 \left(\log \frac{d}{(1-2\eta)\delta\tau} + \log(1/\eps) \log(d \log(1/\eps))\right)\right)$$ labeled examples and $$O((1-2\eta)^{-2} \cdot d^2 \cdot \eps^{-1} \log(1/\eps) \log(d \log(1/\eps)))$$ unlabeled samples.
\end{corollary}

\section{Differentially-private active learning}
\label{sec:privacy}
In this section we show that active SQ learning algorithms can also be used to obtain differentially private active learning algorithms. We assume that a learner has full access to unlabeled portion of some database of $n$ examples $S \subseteq X \times Y$ which correspond to records of individual participants in the database. In addition, for every element of the database $S$ the learner can request the label of that element. As usual, the goal is to minimize the number of label requests. In addition, we would like to preserve the {\em differential privacy} of the participants in the database, a now-standard notion of privacy introduced in \cite{DMNS06}. A simple scenario in which active differentially-private learning could be valuable is medical research in which the goal is to create an automatic predictor of whether a person has certain medical condition. It is often the case that while many unlabeled patient records are available, discovering the label requires work by a medical expert or an expensive test (or both). In such a scenario an active learning algorithm could significantly reduce costs of producing a good predictor of the condition while differential privacy ensures that the predictor does not reveal any information about the patients whose data was used for the algorithm.

Formally, for some domain $X \times Y$, we will call $S \subseteq X \times Y$ a \emph{database}. Databases $S,S'\subset X \times Y$ are \emph{adjacent} if one can be obtained from the other by modifying a single element. Here we will always have $Y = \{-1,1\}$. In the following, $A$ is an algorithm that takes as input a database $D$ and outputs an element of some finite set $R$.
\begin{definition}[Differential privacy \cite{DMNS06}]
A (randomized) algorithm $A:2^{X\times Y} \rightarrow R$ is \emph{$\alpha$-differentially-private} if for all $r \in R$ and every pair of adjacent databases $S,S'$, we have $\Pr[A(S) = r] \leq e^\eps\Pr[A(S') = r]$. \label{def:privacy}
\end{definition}
Here we consider algorithms that operate on $S$ in an active way. That is the learning algorithm receives the unlabeled part of each point in $S$ as an input and can only obtain the label of a point upon request. The total number of requests is the label complexity of the algorithm. We note the definition of differential privacy we use does not make any distinction between the entries of the database for which labels were requested and the other ones. In particular, the privacy of all entries is preserved. Further, in our setting the indices of entries for which the labels are requested are not a part of the output of the algorithm and hence do not need to be differentially private.

As first shown by Blum \etal \cite{BlumDMN:05}, SQ algorithms can be automatically translated into differentially-private\footnote{In \cite{BlumDMN:05} a related but different definition of privacy was used. However, as pointed out in \cite{KasiviswanathanLNRS11} the same translation can be used to achieve differential privacy.} algorithms. We now show that, analogously, active SQ learning algorithms can be automatically transformed into differentially-private active learning algorithms.
\begin{theorem}
\label{active-private}
Let $A$ be an algorithm that learns a class of functions $H$ to accuracy $1-\eps$ over distribution $D$ using $M_1$ active SQs of tolerance $\tau$ and filter tolerance $\tau_0$ and $M_2$ target-independent queries of tolerance $\tau_u$. There exists a learning algorithm $A'$ that given $\alpha > 0,\delta >0$ and active access to database $S \subseteq X \times \{-1,1\}$ is $\alpha$-differentially-private and uses at most $O([\frac{M_1}{\alpha\tau} + \frac{M_1}{\tau^2}]\log(M_1/\delta))$ labels. Further, for some $n=O([\frac{M_1}{\alpha\tau_0\tau} +
\frac{M_1}{\tau_0\tau^2}+\frac{M_2}{\alpha\tau_u} +
\frac{M_2}{\tau_u^2}]\log((M_1+M_2)/\delta))$, if $S$ consists of at least $n$ examples drawn randomly from $D$ then with, probability at least $1-\delta$, $A'$ outputs a hypothesis with accuracy $\geq 1-\eps$ (relative to distribution $D$). The running time of $A'$ is the same as the running time of $A$ plus $O(n)$.
\end{theorem}
\begin{proof}
%The proof follows the lines of \cite{DN04,BDMN05,KLNRS08}.
We first consider the active SQs of $A$. We will answer each such query
using a disjoint set of $O([\frac{1}{\alpha\tau_0\tau} +
\frac{1}{\tau_0\tau^2}]\log(M_1/\delta))$ unlabeled examples.
The subset $T$ of these examples that satisfy the filter will be
queried for their labels and used to compute an answer to the
statistical query (by taking the empirical average in the usual way).
Additional noise drawn from a Laplace distribution will
then be added to the answer in order to preserve privacy.

We begin by analyzing the amount of noise needed to achieve the desired privacy
guarantee.  First, since each query is being answered using a disjoint set of
examples, changing any given example can affect the answer to at most
one query; so, it suffices to answer each query with $\alpha$-differential privacy.
Second, modifying any given example can change the empirical answer to
a query by at most $1/|T|$ (there are three cases: the example was
already in $T$ and remains in $T$ after modification, the example was
in $T$ and is removed from $T$ due to the modification, or the
example was not in $T$ and is added to $T$ due to the modification;
each changes the empirical answer by at most $1/|T|$).
Therefore, $\alpha$-privacy can be achieved by adding a quantity $\xi$
selected from a Laplace distribution of width $O(\frac{1}{\alpha |T|})$ to the
empirical answer over the labeled sample.
Finally, we solve for the size of $T$ needed to ensure that with sufficiently high probability $|\xi| \leq
\tau/2$ so that the effect on the active SQ after correction for noise
is at most $\tau/2$.
Specifically, the Laplace distribution has the property that with
probability at least $1-\delta'$, the magnitude of $\xi$ is at most
$O(\frac{1}{\alpha |T|}\log(1/\delta'))$.  Setting this to
$\tau/2$ and using $\delta' = \delta/(6M_1)$ we have that
privacy $\alpha$ can be guaranteed with perturbation at most
$\tau/2$ as long as we have
$|T| \geq c(\frac{1}{\alpha\tau}\log(M_1/\delta))$ for
sufficiently large constant $c$.

Next, we also need  to ensure that $T$ is large enough so that with probability at
least $1-\delta'$, even without the added Laplace noise, the empirical average
of the query function over $T$ is within $\tau/2$ of the true
value.  By Hoeffding bounds, this is ensured if
$|T| \geq c(\frac{1}{\tau^2}\log(M_1/\delta))$ for sufficiently
large constant $c$.

Finally, we need the unlabeled sample to be large enough so that with
probability at least $1-\delta'$, the labeled sample $T$ will satisfy
both the above conditions. By Hoeffding bounds, this is ensured by an
unlabeled sample of size $O(\frac{1}{\tau_0}[\frac{1}{\alpha\tau} +
\frac{1}{\tau^2}]\log(M_1/\delta))$.

The above analysis was for each active SQ.  There are $M_1$
active SQs in total so the total sample size is a factor $M_1$ larger,
and by a union bound over all $M_1$ queries we have that with
probability at least  $1 - \delta/2$, all are answered within their
desired tolerance levels.

Now, we analyze the $M_2$ target-independent queries. Here, by standard analysis (which is also a special case of the analysis above), we get that it is sufficient to use $O([\frac{1}{\alpha \tau_u} +
\frac{1}{\tau_u^2}]\log(M_2/\delta))$ unlabeled samples  to answer all the queries with probability at least $1-\delta/2$.
Finally, summing up the sample sizes and applying a union bound over the failure probabilities we get the claimed bounds on the sample complexity and running time. We remark that the algorithm is $\alpha$-differentially-private even when the samples are not drawn from distribution $D$.
\end{proof}

The simulation above can be easily made tolerant to random classification (or uncorrelated) noise in exactly the same way as in Theorem \ref{th:simulate-noise}.

In our setting it is also natural to treat the privacy of labeled and unlabeled parts differently. For much the same reason that unlabeled data is often much more plentiful than labeled data, in many cases the label information
is much more sensitive, in a privacy sense, than the unlabeled
feature vector. For example the unlabeled data may be fully public
(obtained by crawling the web or from a public address-book) and the
labels obtained from a questionnaire. To reflect this one can define two privacy parameters $\alpha_\ell$ and
$\alpha$ with $\alpha_\ell$ denoting the (high) sensitivity of the labeled
information and $\alpha$ denoting the (lower) sensitivity of the feature
vector alone. More formally, in addition to requiring $\alpha$-differential privacy we can require $\alpha_\ell$-differential privacy on databases which differ only in a single label (for $\alpha_\ell < \alpha$). A special case of this model where only label privacy matters was studied in \cite{ChaudhuriHsu:11} (a model with a related but weaker requirement in which labeled points are private and unlabeled are not was recently considered in \cite{claire-privacy-13}). It is not hard to see that with this definition our analysis will give an algorithm that uses $O([\frac{M_1}{\alpha_\ell \tau} + \frac{M_1}{\tau^2}]\log(M_1/\delta))$ labels and requires a database of size $n$ for some $n=\Omega([\frac{M_1}{\alpha_\ell \tau_0 \tau} + \frac{M_1}{\tau_0\tau^2}+\frac{M_2}{\alpha \tau_u} +
\frac{M_2}{\tau_u^2}]\log((M_1+M_2)/\delta))$. Note that in this result the privacy constraint on labels does not affect the number of samples required to simulate target-independent queries.

%Kasiviswanathan \etal showed that SQ algorithms can automatically transformed to {\em locally differentially private} algorithms, namely algorithms in which  \cite{KasiviswanathanLNRS11}.

\paragraph{Improvement over passive differentially-private learning}
An immediate consequence of Theorem~\ref{active-private} is that for learning of homogeneous halfspaces over uniform or log-concave distributions we can obtain differential privacy while essentially preserving the label complexity. For example, by combining Theorems ~\ref{active-private} and~\ref{th:learn-hs-asq}, we can efficiently and differentially-privately learn homogeneous halfspaces under the uniform distribution with privacy parameter $\alpha$ and error parameter $\epsilon$ by using only $\tilde{O}(d \sqrt{d} \log(1/\epsilon))/\alpha + d^2 \log(1/\epsilon)) $ labels. However, it is known that any passive learning algorithm, even ignoring privacy considerations and noise requires $\Omega\left(d/\epsilon\right)$ labeled examples~\cite{phil}. So for $\alpha \geq 1/\sqrt{d}$ and small enough $\epsilon$ we get better label complexity.

\section{Discussion}
\label{sec:discussion}
We described a framework for designing efficient active learning algorithms that are tolerant to random classification noise. We used our framework to obtain the first  computationally-efficient algorithm for actively learning homogeneous linear separators over log-concave distributions with exponential improvement in the dependence on the error $\eps$ over its passive counterpart. In addition, we showed that our algorithms can be automatically converted to efficient active differentially-private algorithms.

Our work suggests that, as in passive learning, active statistical algorithms might be essentially as powerful as example-based efficient active learning algorithms. It would be interesting to find more general evidence supporting this claim or, alternatively, a counterexample. An important aspect of (passive) statistical learning algorithms is that it is possible to prove unconditional lower bounds on such algorithms using SQ dimension~\cite{BlumFJ+:94} and its extensions. It would be interesting to develop an active analogue of these techniques and give meaningful lower bounds based on them. This could provide a useful tool for understanding the sample complexity of differentially private active learning algorithms.
% Use it in an interesting way.
%\begin{enumerate}
%\item What is the unlabeled sample complexity in Lemma 2.2 if the queries are *adaptive*. This is potentially relevant for improving the sample complexity in Theorem 5.2. Also seems fundamental SQ question.
%\item Active SQ dimension. Use it in an interesting way.
%%\end{enumerate}

\subsection*{Acknowledgments}
We thank Avrim Blum and Santosh Vempala for useful discussions.
This work was supported in part by NSF grants CCF-0953192, CCF-110128, and CCF 1422910, AFOSR grant FA9550-09-1-0538, ONR grant N00014-09-1-0751, and a Microsoft Research Faculty Fellowship.

\bibliography{active-sq}
\appendix
\section{Passive SQ learning of halfspaces}
\label{app-dv}
\newcommand{\taudv}{\tau_{\mathtt{DV}}}
\newcommand{\sqhsdv}{{\tt LearnHS-DV}}
The first SQ algorithm for learning general halfspaces was given by Blum \etal \cite{BlumFKV:97}. This algorithm requires access to unlabeled samples from the unknown distribution and therefore is only label-statistical. This algorithm can be used as a basis for our active SQ algorithm but the resulting active algorithm will also be only label-statistical. As we have noted in Section \ref{sec:model}, this is sufficient to obtain our RCN tolerant active learning algorithm given in Cor.~\ref{cor:learn-hs-rcn-noise-log-concave}. However our differentially-private simulation needs the algorithm to be (fully) statistical. Therefore we base our algorithm on the algorithm of Dunagan and Vempala for learning halfspaces \cite{DunaganVempala:04}. While \cite{DunaganVempala:04} does not contain an explicit statement of the SQ version of the algorithm it is known and easy to verify that the algorithm has a SQ version \cite{Vempala:13pc}. This follows from the fact that the algorithm in \cite{DunaganVempala:04} relies on a combination on the Perceptron \cite{Rosenblatt:58} and the modified Perceptron algorithms \cite{BlumFKV:97} both of which have SQ versions \cite{Bylander:94,BlumFKV:97}. Another small issue that we need to take care of to apply the algorithm is that the running time and tolerance of the algorithm depend polynomially (in fact, linearly) on $\log(1/\rho_0)$, where $\rho_0$ is the {\em margin} of the points given to the algorithm. Namely, $\rho_0 = \min_{x \in S}\frac{|w \cdot x|}{\|x\|}$, where $h_w$ is the target homogeneous halfspace and $S$ is the set of points given to the algorithm. We are dealing with continuous distributions for which the margin is 0 and therefore we make the following observation. In place of $\rho_0$ we can use any margin $\rho_1$ such that the probability of being within margin $\leq \rho_1$ around the target hyperplane is small enough that it can be absorbed into the tolerance of the statistical queries of the Dunagan-Vempala algorithm for margin $\rho_1$. Formally,
\begin{definition}
\label{def:min-margin}
For positive $\delta<1$ and distribution $D$, we denote $$\gamma(D,\delta) = \inf_{\|w\| =1}\sup_{\gamma>0} \left\{\gamma \ \left|\ \Pr_D\left[\frac{|w \cdot x|}{\|x\|}\leq \gamma\right] \leq \delta\ \right.\right\},$$ namely the smallest value of $\gamma$ such that for every halfspace $h_w$, $\gamma$ is the largest such that the probability of being within margin $\gamma$ of $h_w$ under $D$ is at most $\delta$. Let $\taudv(\rho,\eps)$ be the tolerance of the SQ version of the Dunagan-Vempala algorithm when the initial margin is equal to $\rho$ and error is set to $\eps$. Let $$\rho_1(D,\eps) = \frac{1}{2} \sup_{\rho \geq 0}\{\rho \cond \gamma(D,\taudv(\rho,\eps/2)/3) \geq  \rho \}.$$
\end{definition}
Now, for $\rho_1 = \rho_1(D,\eps)$ we know that $\gamma(D,\taudv(\rho,\eps/2)/3) \geq  \rho$. Let $D'$ be defined as distribution $D$ conditioned on having margin at least  $\rho$ around the target hyperplane $h_w$. By the definition of the function $\gamma$, the probability of being within margin $\leq \rho$ is at most $\taudv(\rho,\eps/2)/3$.  Therefore for any query function $g:X\times \on \rightarrow [-1,1]$, $$\left|\E_D[g(x,h_w(x))] - (1-\taudv(\rho,\eps/2)/3)\E_{D'}[g(x,h_w(x))] \right| \leq \taudv(\rho,\eps/2)/3$$ and hence $|\E_D[g(x,h_w(x))] - \E_{D'}[h_w(x,f(x))]| \leq 2 \taudv(\rho,\eps/2)/3$. This implies that we can obtain an answer to any SQ relative to $D'$ with tolerance $\taudv(\rho,\eps/2)$ by using the same SQ relative to $D$ with tolerance $\taudv(\rho,\eps/2)/3$. This means that by running the Dunagan-Vempala algorithm in this way we will obtain a hypothesis with error at most $\eps/2$ relative to $D'$. This hypothesis has error at most $\eps/2 + 2 \taudv(\rho,\eps/2)/3$ which, without loss of generality, is at most $\eps$. Combining these observations about the Dunagan-Vempala algorithm, we obtain the following statement.
\begin{theorem}[\cite{DunaganVempala:04}]
\label{thm:dv}
There exists a SQ algorithm \sqhsdv\ that learns $\cH_d$ to accuracy $1-\eps$ over any distribution $D$. Further \sqhs\ outputs a homogeneous halfspace, runs in time polynomial in $d$,$1/\eps$ and $\log(1/\rho_1)$ and uses SQs of tolerance $ \geq 1/\poly(d,1/\eps,\log(1/\rho_1))$, where $\rho_1 = \rho_1(D,\eps)$.
\end{theorem}
To apply Theorem \ref{thm:dv} we need to obtain bounds on $\rho_1(D,\eps)$ for any distribution $D$ on which we might run the Dunagan-Vempala algorithm.

\begin{lemma}
\label{lem:log-concave-margin}
Let $D$ be an isotropic log-concave distribution. Then for any $\delta \in (0,1/20)$, $\gamma(D,\delta) \geq \delta/(6 \ln(1/\delta))$.
\end{lemma}
\begin{proof}
Let $\gamma \in (0,1/16)$ and $w$ be any unit vector. We first upper-bound $\Pr_D\left[\frac{|w \cdot x|}{\|x\|}\leq \gamma\right]$.
\alequ{\Pr_D\left[\frac{|w \cdot x|}{\|x\|}\leq \gamma\right] &\leq \Pr_D\left[\|x\|\leq \ln(1/\gamma) \mbox{ and }\frac{|w \cdot x|}{\|x\|}\leq \gamma\right] + \Pr_D\left[ \|x\| > \ln(1/\gamma) \right] \nonumber\\
& \leq \Pr_D\left[|w \cdot x| \leq \gamma \cdot \ln(1/\gamma) \right] + \Pr_D\left[ \|x\| > \ln(1/\gamma) \right] \label{eq:margin-weight}.}
By Lemma 5.7 in \cite{LovaszVempala:07}, for an isotropic log-concave $D$ and any $R > 1$, $\Pr_D[\|x\| > R] <  e^{-R+1}$.
Therefore $$\Pr_D\left[ \|x\| > \ln(1/\gamma) \right] \leq e \cdot \gamma. $$
Further, by Lemma \ref{lemma:logc-project},
$$ \Pr_D\left[|w \cdot x| \leq \gamma \cdot \ln(1/\gamma) \right] \leq 2 \gamma \cdot \ln(1/\gamma).$$
Substituting, these inequalities into eq.~(\ref{eq:margin-weight}) we obtain that for $\gamma \in (0,1/16)$,
$$\Pr_D\left[\frac{|w \cdot x|}{\|x\|}\leq \gamma\right] \leq 2 \gamma \cdot \ln(1/\gamma) + e \cdot \gamma
\leq 3 \gamma \cdot \ln(1/\gamma) .$$

This implies that for $\gamma = \delta/(6 \ln(1/\delta))$ and any unit vector $w$,
$$\Pr_D\left[\frac{|w \cdot x|}{\|x\|}\leq \gamma\right] \leq 3 \delta/(6 \ln(1/\delta)) \cdot \left(\ln(1/\delta) + \ln(6 \ln(1/\delta)\right) \leq \delta,$$
where we used that for $\delta < 1/20$, $6 \ln(1/\delta) \leq 1/\delta$.
By definition of $\gamma(D,\delta)$, this implies that $\gamma(D,\delta) \geq \delta/(6 \ln(1/\delta))$.
\end{proof}

We are now ready to prove Theorem \ref{thm:dv-logc}.
\label{thm:dv-logc}
There exists a SQ algorithm \sqhs\ that learns $\cH_d$ to accuracy $1-\eps$ over any distribution $D_{|\chi}$, where $D$ is an isotropic log-concave distribution and $\chi:\R^d \rightarrow [0,1]$ is a filter function. Further \sqhs\ outputs a homogeneous halfspace, runs in time polynomial in $d$,$1/\eps$ and $\log(1/\lambda)$ and uses SQs of tolerance $ \geq 1/\poly(d,1/\eps,\log(1/\lambda))$, where $\lambda = \E_D[\chi(x)]$.
\begin{proof}[Proof of Thm.~\ref{thm:dv-logc}]
To prove the theorem we bound $\rho_1 = \rho_1(D_{|\chi},\eps)$ and then apply Theorem \ref{thm:dv}. We first observe that for any event $\Lambda$,
$$\pr_{D_{|\chi}}[\Lambda] \leq \pr_D[\Lambda]/\E_D[\chi].$$
Applying this to the event $\frac{|w \cdot x|}{\|x\|} \leq \gamma$ in Definition \ref{def:min-margin} we obtain that $\gamma(D_{|\chi},\delta) \geq \gamma(D,\delta \cdot \E_D[\chi])$. By Lemma \ref{lem:log-concave-margin}, we get that $\gamma(D_{|\chi},\delta) = \Omega(\lambda \delta/\log(1/(\lambda \delta)))$.

In addition, by Theorem \ref{thm:dv}, $\taudv(\rho,\eps) \geq 1/p(d,1/\eps,\log(1/\rho))$ for some polynomial $p$. This implies that $$\gamma(D_{|\chi},\taudv(\rho,\eps/2)/3) \leq \gamma \left(D_{|\chi},\Omega\left(\frac{1}{p(d,1/\eps,\log(1/\rho))}\right)\right) =
 \tilde{\Omega}\left(\frac{\lambda}{p(d,1/\eps,\log(1/\rho))}\right).$$
Therefore, we will obtain that,
$$\rho_1(D_{|\chi}, \eps) =\tilde{\Omega}\left(\frac{\lambda}{p(d,1/\eps,1)}\right) .$$
\end{proof}
By plugging this bound into Theorem \ref{thm:dv} we obtain the claim.

\section{Proofs from Section \ref{sec:uniform}}
\label{app-proofs-uniform}
We now prove Lemmas \ref{lem:define-cp} and \ref{lem:cp-deriv} which we restate for convenience.

\begin{lemma}[Lem.~\ref{lem:define-cp} restated]
For any $v,w\in S_{d-1}$ such that $\|v-w\| = \Delta \leq \sqrt{2}$ and $\gamma > 0$, $$\pr [h_v(x) \neq h_w(x) \cond |\la v , x \ra | \leq \gamma] = \frac{A_{d-3} \int_0^\gamma (1-r^2)^{(d-3)/2}\int_{\frac{r \cdot \sqrt{2-\Delta^2}}{\Delta \cdot \sqrt{1-r^2}}}^1  (1-s^2)^{(d-4)/2}ds\cdot dr }{A_{d-2} \int_0^\gamma  (1-r^2)^{(d-3)/2} dr}. $$ We denote the probability by $\cp_d(\gamma,\Delta)$.
\end{lemma}
\begin{proof}
By using spherical symmetry, we can assume without loss of generality that $v = (1,0,0,\ldots,0)$ and $w=(\sqrt{1-\Delta^2/2}, \Delta/\sqrt{2},0,0,\ldots,0)$.
We now examine the surface area of the points that satisfy $h_w(x) = -1$ and $0 \leq \la v, x \ra \leq \gamma$ (which is a half of the error region at distance at most $\gamma$ from $v$). To compute it we consider the points on $S_{d-1}$ that satisfy $\la v, x \ra = r$. These points form a hypersphere $\sigma$ of dimension $d-2$ and radius $\sqrt{1-r^2}$. In this hypersphere points that satisfy $h_w(x) = -1$ are points $(r,s,x_3,..,x_d) \in S_d$ for which $r\sqrt{1-\Delta^2/2} + s \Delta/\sqrt{2} \leq 0$. In other words, $s \geq r\sqrt{2-\Delta^2}/\Delta$ or points of $\sigma$
which are at least $r\sqrt{2-\Delta^2}/\Delta$ away from hyperplane $(0,1,0,0,\ldots,0)$ passing through the origin of $\sigma$ (also referred to as a hyperspherical cap). As in the equation (\ref{eq:hs-band}), we obtain that its $d-2$-dimensional surface area is:
$$ (1-r^2)^{(d-2)/2}\int_{\frac{r \cdot \sqrt{2-\Delta^2}}{\Delta \cdot \sqrt{1-r^2}}}^1 A_{d-3} (1-s^2)^{(d-4)/2}ds$$
Integrating over all $r$ from $0$ to $\gamma$ gives the surface area of the region $h_w(x) = -1$ and $0 \leq \la v , x \ra  \leq \gamma$:
$$\int_0^\gamma (1-r^2)^{(d-3)/2}\int_{\frac{r \cdot \sqrt{2-\Delta^2}}{\Delta \cdot \sqrt{1-r^2}}}^1 A_{d-3} (1-s^2)^{(d-4)/2}ds\cdot dr .$$
Hence the conditional probability is as claimed.
\end{proof}

\begin{lemma}[Lem.~\ref{lem:cp-deriv} restated]
For $\Delta \leq \sqrt{2}$, any $d\geq 4$, and $\gamma \geq \Delta/(2\sqrt{d})$, $\partial_\Delta \cp_d(\gamma,\Delta) \geq 1/(56\gamma \cdot \sqrt{d})$.
\end{lemma}
\begin{proof}
First note that $$\tau(\gamma) = \frac{A_{d-3}}{A_{d-2} \int_0^\gamma  (1-r^2)^{(d-3)/2} dr} $$ is independent of $\Delta$ and therefore it is sufficient to differentiate $$\theta(\gamma,\Delta) = \int_0^\gamma (1-r^2)^{(d-3)/2}\int_{\frac{r \cdot \sqrt{2-\Delta^2}}{\Delta \cdot \sqrt{1-r^2}}}^1  (1-s^2)^{(d-4)/2}ds\cdot dr .$$
Let $\gamma' = \Delta/(2\sqrt{d})$ (note that by our assumption $\gamma' \leq \gamma$). By the Leibnitz integral rule, \alequn{\partial_\Delta \theta(\gamma,\Delta) &=
\int_0^\gamma (1-r^2)^{(d-3)/2} \partial_\Delta\left(\int_{\frac{r \cdot \sqrt{2-\Delta^2}}{\Delta \cdot \sqrt{1-r^2}}}^1  (1-s^2)^{(d-4)/2}ds \right) \cdot dr \\
&= \int_0^\gamma (1-r^2)^{(d-3)/2} \left(1- \frac{r^2 (2-\Delta^2)}{\Delta^2(1-r^2)}\right)^{\frac{d-4}{2}} \cdot \frac{2r}{\Delta^2 \sqrt{1-r^2} \sqrt{2-\Delta^2}} \cdot dr \\ & \geq
\int_0^\gamma (1-r^2)^{(d-4)/2} \left(1- \frac{2 r^2}{\Delta^2(1-r^2)}\right)^{\frac{d-4}{2}} \cdot \frac{2r}{\sqrt{2} \Delta^2} \cdot dr \\ & \geq
\int_0^{\gamma'} (1-r^2)^{(d-4)/2} \left(1- \frac{2 r^2}{\Delta^2(1-r^2)}\right)^{\frac{d-4}{2}} \cdot \frac{\sqrt{2} \cdot r}{\Delta^2} \cdot dr .}
Now using the conditions $\Delta \leq \sqrt{2}$, $d\geq 4$, we obtain that $\gamma' \leq 1/(2\sqrt{2})$ and hence for all $r \in [0,\gamma']$, $1-r^2 \geq 7/8$ and $r^2/\Delta^2 \leq \gamma'^2/\Delta^2 = 1/(4d)$. This implies that for all $r\in[0,\gamma']$, $$1- \frac{2 r^2}{\Delta^2(1-r^2)} \geq 1- \frac{2}{\frac{7}{8} 4d} = 1-\frac{4}{7d}\ .$$ Now, $$\left(1-\frac{4}{7d}\right)^{(d-4)/2} \geq 1- \frac{4 (d-4)}{14 d} \geq \frac{5}{7}.$$

Substituting this into our expression for $\partial_\Delta \theta(\gamma,\Delta)$ we get
\alequn{
\partial_\Delta \theta(\gamma,\Delta) &  \geq
\int_0^{\gamma'} (1-r^2)^{(d-4)/2} \frac{\sqrt{2} \cdot 5r}{7 \Delta^2} \cdot dr \geq
\frac{1}{\Delta^2}\int_0^{\gamma'} (1-r^2)^{(d-4)/2} \cdot r \cdot dr \\
& = \frac{1}{\Delta^2(d-2)} \left(1 - (1-\gamma'^2)^{(d-2)/2}\right)\geq \frac{1}{\Delta^2(d-2)} \left(1 - e^{-\gamma'^2 (d-2)/2}\right) \\ & \geq^{(*)}  \frac{1}{\Delta^2(d-2)} \left(1 - (1-\frac{(d-2)\gamma'^2}{4})\right) = \frac{\gamma'^2}{4 \Delta^2} = \frac{1}{16d},
}
where to derive $(*)$ we use the fact that $e^{-\gamma'^2 (d-2)/2} \leq 1-\gamma'^2 (d-2)/4$ since $e^{-x} \leq 1-x/2$ for every $x \in [0,1]$ and $\gamma'^2 (d-2)/2 \leq \frac{\Delta^2 (d-2)}{8d} < 1$.
At the same time, $\int_0^\gamma  (1-r^2)^{(d-3)/2} dr \leq \gamma$ and therefore,

$$\partial_\Delta \cp_d(\gamma,\Delta) = \tau(\gamma) \cdot \partial_\Delta \theta(\gamma,\Delta) \geq
\frac{A_{d-3}}{16 d \gamma A_{d-2}} \geq \frac{1}{32 \sqrt 3 \gamma \sqrt{d} } > \frac{1}{56 \gamma \sqrt{d} } ,$$
where we used that $\frac{A_{d-3}}{A_{d-2}} \geq \sqrt{d}/(2\sqrt{3})$ (\eg\cite{DasguptaKM:09}).
\end{proof}

\end{document}